\documentclass[11pt]{article}

\usepackage[margin=1in]{geometry}
\usepackage{amsmath,amssymb,amsthm,mathtools}
\usepackage{booktabs,array,enumitem,microtype}
\usepackage{xcolor}
\usepackage{graphicx}
\usepackage{algorithm}
\usepackage{algpseudocode}
\usepackage[round,authoryear]{natbib}
\usepackage[hidelinks]{hyperref}
\hypersetup{pdftitle={First-Order Algorithms for Online Resource Allocation},pdfauthor={Menglong Li and Jiawei Zhang}}
\allowdisplaybreaks

\numberwithin{equation}{section}

\newtheorem{theorem}{Theorem}[section]
\newtheorem{lemma}[theorem]{Lemma}
\newtheorem{proposition}[theorem]{Proposition}
\newtheorem{corollary}[theorem]{Corollary}
\newtheorem{assumption}{Assumption}
\theoremstyle{definition}

\newtheorem{remark}[theorem]{Remark}

\DeclareMathOperator*{\argmax}{arg\,max}
\DeclareMathOperator{\dist}{dist}
\newcommand{\R}{\mathbb R}
\newcommand{\E}{\mathbb E}
\newcommand{\Prob}{\mathbb P}
\newcommand{\one}{\mathbf 1}
\newcommand{\OPT}{\operatorname{OPT}}
\newcommand{\REV}{\operatorname{REV}}
\newcommand{\Reg}{\operatorname{Reg}}
\newcommand{\cP}{\mathcal P}
\newcommand{\cS}{\mathcal S}
\newcommand{\norm}[1]{\left\lVert#1\right\rVert}
\newcommand{\pos}[1]{\left[#1\right]_+}
\newcommand{\PG}{\mathcal T}

\title{A First-Order Learning Algorithm for Online Resource Allocation with Constant Regret}
\author{
Menglong Li\thanks{Department of Decision Analytics and Operations, College of Business, City University of Hong Kong, Hong Kong.}
\qquad\qquad
Jiawei Zhang\thanks{Department of Technology, Operations, and Statistics, Leonard N. Stern School of Business, New York University, New York, New York 10012.}
}

\date{October 13, 2026}

\begin{document}
\maketitle

\begin{abstract}
We study a finite-horizon online resource allocation problem with initial
resource capacities proportional to the horizon. In each period, a request
type is observed and one action is chosen from a finite menu. Each action earns
a reward and consumes a vector of resources. We present a primal first-order
learning policy that, in each period, performs one gradient ascent update of
the action coordinates associated with one request type, for a fluid relaxation
in which resource overuse incurs a quadratic penalty. The policy achieves $O(1)$ expected additive regret relative to the
hindsight optimum, with a bound independent of the horizon $T$, under a general
arrival model. This model includes independent and identically distributed
arrivals with unknown probabilities, independent nonstationary arrivals with
known probabilities, independent nonstationary arrivals with unknown
probabilities and one historical sample path, and Markov-modulated arrivals
with an unknown Markov chain. The policy does not solve any linear program, and
the regret bound does not require a nondegeneracy assumption on the fluid linear
program.
\end{abstract}

\section{Introduction}

We study a finite-horizon online resource allocation problem. There are
multiple resources with fixed initial capacities. In each period, one request
arrives, its type is observed, and the decision maker chooses one action from
a finite type-dependent menu. An action earns a reward and consumes a vector
of resources; rejection earns no reward and consumes no resources. Decisions
are immediate and irrevocable, and the capacity constraints must hold on every
sample path. The objective of the decision maker is to maximize the expected
total reward over the horizon. The request types form a random sequence over a finite set. Its
distribution may change over time, and it may be unknown to the decision maker.
The model includes quantity-based network revenue management
\citep{gallego_van_ryzin_1997}, online order fulfillment
\citep{acimovic_graves_2015,jasin_sinha_2015}, and display ad allocation with
guaranteed contracts
\citep{feldman_henzinger_korula_mirrokni_stein_2010,vee_vassilvitskii_shanmugasundaram_2010}
as special cases. We measure the performance of a policy by its expected
additive regret, which is the expected reward of the hindsight optimum that
observes the full arrival sequence minus the expected reward of the policy.

For quantity-based network revenue management (NRM), the classical static bid-price
control has $O(\sqrt T)$ regret under proportional scaling, i.e., the initial
capacities grow proportionally with the horizon \citep{talluri_van_ryzin_1998}.
In their model, the requests in different periods are independent, and their
distributions are known and may change over time.
 By re-optimizing once,
\citet{reiman_wang_2008} obtained an $o(\sqrt T)$ regret bound.
Their arrival processes have known constant rates and satisfy a functional
central limit theorem, which holds, for example, for independent renewal
processes.
\citet{jasin_kumar_2012} then established an $O(1)$ regret bound for a re-solving policy
under a nondegeneracy assumption, which requires that the same optimal basis of the fluid linear program remain optimal
under sufficiently small perturbations of capacities and expected demand.
Their arrivals are Poisson with known rates, and with a suitable schedule, the
policy re-solves the fluid linear program only $O(\log T)$ times.
\citet{bumpensanti_wang_2020} obtained a constant regret bound without nondegeneracy by an infrequent re-solving policy for
quantity-based NRM that solves the fluid linear
program at $O(\log\log T)$ epochs. \citet{vera_banerjee_2021} proposed the
Bayes Selector and proved constant regret for a broad class of finite-type
online allocation problems that includes NRM, online packing and online matching.
Their policy solves a fluid linear program every
period.

Solving a large fluid linear program (LP) repeatedly can be costly in
time-sensitive applications. Recent work has sought to reduce the number of
fluid LP solves while still achieving an $O(1)$ regret bound. \citet{he_wei_xu_yu_2025} study essentially the
same multi-action allocation model as ours. Their policy solves the fluid
LP only once at the beginning and then updates the dual prices in each
period based on resource consumption. They prove an $O(1)$ regret bound under
a general position gap condition equivalent to nondegeneracy.
Similar results are obtained by \citet{gupta_2024} for a multiway matching
model that covers NRM.

Solving the fluid LP at the beginning requires knowledge of the arrival distribution. In contrast, the re-solving algorithm of
\citet{li_wang_zhang_2024} allows the arrival probabilities to be unknown. They
obtain $O(1)$ regret with $O(\log\log T)$ fluid linear program solves using the
observed arrivals. Moreover, their bound does not require nondegeneracy.

Therefore, all of the constant-regret policies described above solve the fluid
LP at least once. In online stochastic matching
\citep{feldman_mehta_mirrokni_muthukrishnan_2009}, constant regret can be attained
without solving an LP if the network satisfies a Hall-type condition.
\citet{asadpour_wang_zhang_2020} prove this for a greedy policy on the
long-chain design, and \citet{xu_zhang_zhang_zhang_2020} extend it to more
general sparse networks. In both papers, the rewards of all requests are equal.
For general models and algorithms that do not solve
LPs, projected subgradient and mirror-descent policies achieve $O(\sqrt T)$
regret under general input distributions
\citep{li_sun_ye_2020,balseiro_lu_mirrokni_2023,jiang_li_zhang_2025}.
\citet{ma_cao_tsang_xia_2025} develop a fast adaptive dual-gradient algorithm
that performs one gradient update per period and achieves
$O(\log^2 T)$ regret. Their analysis relies on a nondegeneracy assumption and on local second-order growth and smoothness conditions;
the latter are designed for continuous reward distributions and generally fail when there are finitely many request types, as in our model.
\citet{gao_et_al_2026} introduce separate first-order procedures for learning
and decision making and establish an $O(\log T)$ bound for the finite-support case when the
 dual optimal solution is unique. This leaves open whether or not an $O(1)$ bound can be achieved for the general
model by a first-order algorithm without solving any LP.

A second question is whether such a policy can also handle arrivals that are
not independent and identically distributed (i.i.d.), which most of the results
above assume. Several papers study other arrival models.
\citet{vera_banerjee_2021} show that their constant regret bound also holds for known
time-varying arrival probabilities and for request types that follow a known
ergodic Markov chain. \citet[Section~3.3]{li_wang_zhang_2024} discuss 
nonstationary arrival probabilities and prove constant regret for their
re-solving algorithm when the arrival probabilities follow a known irreducible and
aperiodic Markov chain, using the stationary type frequencies in the fluid
LP. 
\subsection{Main results}
We present a primal first-order learning policy for the online resource
allocation model described above. The policy achieves $O(1)$ expected additive
regret relative to the hindsight optimum, with a bound independent of the
horizon $T$, under a general arrival model. In this model, the policy uses a
forecast of the proportions of the request types among the remaining requests.
The forecast must change slowly. Moreover, except in a bounded number of
periods in expectation, it must be close to the realized proportions, and the
probability that it assigns to the type of the arriving request must be at
least a fixed positive constant. The model includes the following arrival
models:
\begin{itemize}[leftmargin=*,itemsep=0pt,parsep=0pt,topsep=2pt]
\item i.i.d.\ arrivals with unknown probabilities;
\item independent nonstationary arrivals with known probabilities;
\item independent nonstationary arrivals with unknown probabilities and one
historical sample path;
\item Markov-modulated arrivals with an unknown Markov chain.
\end{itemize}
When the arrival probabilities are
unknown, the forecast is computed from the observed requests and, when one is
available, from the historical sample path.

The policy applies gradient ascent to a penalized problem, which replaces the
capacity constraints of the fluid LP by a quadratic penalty on
resource overuse. Under i.i.d.\ arrivals, with known or unknown probabilities, in
every period, after observing the type of the arriving request, the policy
updates only the coordinates of that type by one projected gradient ascent step
on the penalized problem. Under the other arrival models, it updates the
coordinates of a type drawn at random, independently of the arrivals. Given
the planned resource consumption of the current solution, the update consists
of a few
matrix-vector multiplications with the resource-consumption vectors of the
actions of the updated type and one Euclidean projection onto a simplex. It
takes $O(dK_j+K_j\log K_j)$ arithmetic operations, where $d$ is the number of
resources and $K_j$ is the number of actions of the updated type $j$ other
than rejection, and it does not involve the actions of the other types.

Moreover, the result does not require a nondegeneracy assumption on the fluid
linear program. Nondegeneracy assumptions are common in this literature, but
they may fail in many resource allocation problems. Reward ties among actions, redundant resources,
and binding constraints with nonunique shadow prices can all lead to
degeneracy, and a small change in capacity or demand can then change the
optimal basis. Regret bounds that rely on nondegeneracy do not apply to such
problems, and the regret of the re-solving policy of \citet{jasin_kumar_2012}
can grow as the square root of the horizon when the problem is degenerate
\citep{bumpensanti_wang_2020}; see also \citet{besbes_kanoria_kumar_2025} for
a discussion of this issue.

Our analysis is inspired by the analysis of \citet{vera_banerjee_2021}. A
constant regret bound follows if, in every period, the distance from the
policy's allocation to the optimal solution set of the fluid LP is at most a
small fixed constant. This condition may fail in some periods. We prove that it
fails only in a bounded number of periods in expectation. The proof has three elements. First, the optimal
solution set changes in every period, and we bound its movement by LP
sensitivity \citep{mangasarian_shiau_1987}. Second, we bound the distance to
this set by a contraction argument similar to those for randomized
block-coordinate methods
\citep{nesterov_2012,richtarik_takac_2014,salzo_villa_2022}. Third, because the
update is for the penalized problem, the allocation may overuse resources. We
show that the overuse is of the order of the penalty parameter, and we extend
the LP error bound \citep{mangasarian_shiau_1987,burke_ferris_1993} to
allocations that overuse resources.

We also evaluate the policy numerically on online order fulfillment instances
built on the network, costs, and inventory rule of \citet{ma_fulfillment_2023};
see Section~\ref{sec:numerical-experiments}. With known arrival
probabilities, we compare it with the policy of \citet{vera_banerjee_2021},
which solves an LP in every period, and with the policy of
\citet{he_wei_xu_yu_2025}, which solves an LP only once. When our policy starts
from an optimal solution of the fluid LP, it has the second lowest average
loss at each of the six horizons, after the policy of
\citet{vera_banerjee_2021}. With unknown arrival probabilities, we compare it
with dual first-order policies
\citep{balseiro_lu_mirrokni_2023,ma_cao_tsang_xia_2025}. Our policy has the
lowest average loss at each horizon, and it has a lower loss than every other
policy in at least 27 of the 30 instances at each horizon. On an instance with
about $10^{11}$ type--action pairs, where no LP can be solved, its loss falls
to $0.26\%$ at $T=10^8$, compared with $1.95\%$ for the dual mirror descent
policy of \citet{balseiro_lu_mirrokni_2023}.

The rest of the paper is organized as follows. In Section~\ref{sec:model}, we
formulate the model, the hindsight benchmark, the fluid relaxation, and the two
conditions on the forecast, and we verify these conditions for the arrival
models above. In Section~\ref{sec:algorithm}, we present the policy, state the
main result, and outline its proof. In Section~\ref{sec:preliminary-analysis},
we discuss properties of the fluid LP, including sensitivity bounds and an
error bound that also applies to solutions that overuse resources. The proof of
the regret bound is presented in Section~\ref{sec:regret-analysis}. Then, in
Section~\ref{sec:numerical-experiments}, we compare the policy numerically with
LP-based and first-order policies on online order fulfillment instances with
i.i.d.\ arrivals. Finally, in Section~\ref{sec:concluding-remarks}, we
summarize the results and discuss open questions for related models and
continuous reward distributions. The appendix contains the proofs omitted from
Sections~\ref{sec:model} and~\ref{sec:regret-analysis}.

\section{Model and problem formulation}\label{sec:model}

\subsection{Online resource allocation}\label{subsec:allocation-model}

Consider an online resource allocation problem with $d$ resources and $T$
decision periods. The initial resource capacity is {$B_T=T b$}, where {$b\in\mathbb R_+^d$ is fixed}.
In period $t=1,2,\ldots,T$, a request of type $J_t\in[m]:=\{1,\ldots,m\}$
arrives. Type $j$ is associated with a finite set of actions $\mathcal A_j=\{1,\ldots,K_j\}$. Choosing action $a\in\mathcal A_j$ earns reward $r_{ja}\in\mathbb R_+$ and consumes the resource vector $w_{ja}\in\mathbb R_+^d$. A request can also be rejected, which we denote by action $0$, so the action menu of type $j$ is $\mathcal A_j\cup\{0\}$. Rejection earns no reward and consumes no resources, that is, $r_{j0}=0$ and $w_{j0}=0$.

Let $C_t$ be the remaining capacity at the beginning of period $t$, with
$C_1=B_T$. The decision maker observes the current type and $C_t$ before
choosing an action. A feasible policy $\mu$ chooses the action
$U_t\in\mathcal A_{J_t}\cup\{0\}$ in period $t$ without observing the future types
$J_{t+1},\ldots,J_T$, and it satisfies the capacity constraint
$\sum_{t=1}^{T}w_{J_tU_t}\le B_T$ on every sample path.

The total reward of a feasible policy $\mu$ is denoted by
\begin{equation*}
    \REV_T^\mu
    =\sum_{t=1}^{T}r_{J_tU_t^\mu},
\end{equation*}
where $U_t^\mu$ is the action of $\mu$ in period $t$. We omit the superscript
$\mu$ when the policy is clear from the context.

For a sample arrival sequence $J_1,\ldots,J_T$, define $\OPT_T$ as the total reward of the hindsight optimum, which sees the entire sequence and selects an optimal feasible allocation.
The performance of a policy $\mu$ is measured by its expected additive regret, which is the difference between the expected reward of the hindsight optimum and the expected reward of policy $\mu$,
\begin{equation*}
    \Reg_T(\mu)=\E[\OPT_T]-\E[\REV_T^\mu].
\end{equation*}
We seek a computationally efficient online policy with low additive regret that does not rely on nondegeneracy. To avoid trivial cases, we only consider problems with $\max_{j,a} r_{ja} > 0$ and $\max_{j,a,i} w_{ja,i} > 0$.

For a capacity vector $C\in\R_+^d$ and a demand vector $N\in\R_+^m$, let $\Phi(C,N)$ be the optimal objective value of the linear program (LP)
\begin{eqnarray*}
\Phi(C,N)=&\text{maximize}\quad
&\sum_{j=1}^m\sum_{a\in\mathcal A_j}r_{ja}z_{ja}\\
&\text{subject to}\quad
&\sum_{j=1}^m\sum_{a\in\mathcal A_j}w_{ja}z_{ja}\le C,\\
&&\sum_{a\in\mathcal A_j}z_{ja}\le N_j,
\qquad j\in[m], \\
&&z_{ja}\ge0,
\qquad j\in[m],\ a\in\mathcal A_j.
\end{eqnarray*}
Here, $z_{ja}$ is the number of type $j$ requests assigned to action $a\in\mathcal A_j$, with fractional values allowed, and the remaining $N_j-\sum_{a\in\mathcal A_j}z_{ja}$ requests of type $j$ are rejected. For a fixed arrival sequence, let $N_j=|\{t\in[T]:J_t=j\}|$ be the realized demand of type $j\in[m]$. Every feasible offline allocation is feasible for this LP with $C=B_T$ and the realized demand $N$. Hence, $\OPT_T$ is bounded above by the optimal objective value $\Phi(B_T,N)$ of this LP relaxation, which we call the hindsight LP.

Our algorithm and analysis use a normalized fluid relaxation, in which the
capacities, the demands, and the decision variables of the hindsight LP are
divided by the number of remaining periods. Thus the normalized fluid
relaxation has per-period capacity and demand constraints. Since the
per-period capacity and demand change over time, we define the normalized fluid relaxation for general per-period capacity
$c$ and demand $\eta$.
Let $K=\sum_{j=1}^mK_j$, and identify the coordinates of $\R^K$ with the pairs $(j,a)$, where $j\in[m]$ and $a\in\mathcal A_j$ so that we can represent $x=(x_{ja}) \in \R_+^K$ and $r=(r_{ja}) \in \R_{+}^K$. Then for any $c\in\R_+^d$ and any  $\eta\in\R_+^m$,
the normalized fluid relaxation is
\begin{eqnarray*}
V_r(c, \eta) :=& \text{maximize}\quad &
r^T x \\
& \text{subject to}\quad
& x \in     {\cP(c,\eta)},
\end{eqnarray*}
where \begin{equation*}
    {\cP(c,\eta)}
    :={\left\{x \in \R_+^K:
        \sum_{j=1}^m\sum_{a\in\mathcal A_j}w_{ja}x_{ja}\le c,\quad
        \sum_{a\in\mathcal A_j}x_{ja}\le\eta_j
        \text{ for every }j\in[m]\right\}}.
\end{equation*}
Here, the variable $x_{ja}$ represents the fraction of periods in which a
type $j$ request is assigned to action $a$. For $s>0$, we have $\cP(sc,s\eta)=s\,\cP(c,\eta)$, and therefore
$\Phi(sc,s\eta)=sV_r(c,\eta)$.

\subsection{Arrivals and forecasts}\label{subsec:forecasts}

The types $J_1,\ldots,J_T$ need not be independent or identically distributed,
but their distribution does not depend on the decisions. Before the first
period, the decision maker may know the arrival probabilities or observe a
historical sample path of arrivals. Let $\mathcal F_{t-1}$ be the
$\sigma$-field generated by this information, the types $J_1,\ldots,J_{t-1}$,
and the randomization used by the policy in periods $1,\ldots,t-1$. We assume
that the randomization of the policy is independent of the arrivals and of the
information known before the first period.

Let $s_t:=T-t+1$ be the number of remaining requests in period $t$, including
the current request, and let $\nu_t:=N_t/s_t$, where $N_{t,j}$ is the number
of type $j$ requests among $J_t,\ldots,J_T$. By the scaling above, the
hindsight LP for the remaining periods has the value
$\Phi(C_t,N_t)=s_tV_r(C_t/s_t,\nu_t)$, so the fluid LP with the per-period
capacity $C_t/s_t$ and the demand $\nu_t$ is this hindsight LP normalized by
$s_t$. Since $\nu_t$ depends on the future requests, our policy uses a forecast
of $\nu_t$ that is computed before $J_t$ is observed. Formally, a forecast is a
sequence of probability vectors $\eta_1,\ldots,\eta_T\in\R_+^m$ such that
$\eta_t$ is $\mathcal F_{t-1}$-measurable for every $t\in[T]$. Let
$\ell_t:=\min\{t,s_t\}$. We assume that the forecast satisfies the following
two conditions.

\begin{assumption}\label{ass:forecast-change}
There is a constant $D\ge0$ such that, on every sample path,
\begin{equation*}
    \norm{\eta_{t+1}-\eta_t}_2\le\frac{D}{\ell_t},
    \qquad t=1,\ldots,T-1.
\end{equation*}
\end{assumption}

\begin{assumption}\label{ass:forecast-accuracy}
There is a constant $p_0>0$ such that, for every $\rho>0$,
\begin{equation*}
    \sum_{t=1}^T
    \Prob\bigl(
        \norm{\eta_t-\nu_t}_2>\rho
        \ \text{or}\
        \eta_{t,J_t}<p_0
    \bigr)
    \le\Lambda(\rho),
\end{equation*}
where $\Lambda(\rho)$ is finite and does not depend on $T$.
\end{assumption}

Assumption~\ref{ass:forecast-change} requires the forecast to change by at
most $D/t$ per period in the first half of the horizon and by at most $D/s_t$
per period in the second half. Assumption~\ref{ass:forecast-accuracy} requires
that, except in a bounded number of periods in expectation, the forecast be
within distance $\rho$ of $\nu_t$ and assign probability at least $p_0$ to the
type of the arriving request.

\subsection{Arrival models}\label{subsec:arrival-models}

We now describe four arrival models that have been studied in the literature
and a forecast for each of them. Proposition~\ref{prop:arrival-models} below
shows that these forecasts satisfy Assumptions~\ref{ass:forecast-change}
and~\ref{ass:forecast-accuracy}.

\medskip\noindent\emph{Model 1: i.i.d.\ arrivals with unknown probabilities.}
The types $J_1,\ldots,J_T$ are i.i.d.\ with
$\Prob\{J_t=j\}=p_j>0$ for $j\in[m]$, and the decision maker does not know
$p=(p_1,\ldots,p_m)$. Let $p_{\min}:=\min_{j\in[m]}p_j$. In period $t$,
let $H_{t,j}$ be the number of type $j$ requests observed before period $t$. In
this model, the forecast $\eta_t$ we use is $\pi_t$, defined as
\begin{equation}\label{eq:pi-estimate}
    \pi_{t,j}=\frac{H_{t,j}+1}{t-1+m},
    \qquad j\in[m],
\end{equation}
which is the empirical distribution of the past requests with one additional
request of each type. The vector $\pi_t$ is strictly positive and sums to one.
This is the model of \citet{li_wang_zhang_2024}.

\medskip\noindent\emph{Model 2: Independent arrivals with known probabilities.}
The types $J_1,\ldots,J_T$ are independent, and the decision maker knows the
probability vectors $p_t=(p_{t,1},\ldots,p_{t,m})$ with
$p_{t,j}:=\Prob\{J_t=j\}$. The vectors $p_t$ may change arbitrarily over time,
but there is a constant $p_{\min}>0$, independent of $T$, such that
$p_{t,j}\ge p_{\min}$ for all $t$ and $j$. In this model, the forecast
$\eta_t$ we use is $\bar p_t$, defined as
\begin{equation}\label{eq:known-forecast}
    \bar p_t=\frac1{s_t}\sum_{\tau=t}^Tp_\tau,
\end{equation}
which is the expected value of $\nu_t$. This model includes i.i.d.\ arrivals
with known probabilities
\citep{jasin_kumar_2012,bumpensanti_wang_2020,he_wei_xu_yu_2025},
for which $\bar p_t=p$, and nonstationary arrivals
\citep{talluri_van_ryzin_1998,ma_rusmevichientong_sumida_topaloglu_2020,vera_banerjee_2021,li_wang_zhang_2024}.

\medskip\noindent\emph{Model 3: Independent arrivals with unknown probabilities
and one historical sample path.}
The types $J_1,\ldots,J_T$ are independent with
$\Prob\{J_t=j\}=p_{t,j}\ge p_{\min}>0$ as in Model 2, but the decision maker
does not know the vectors $p_t$. Instead, before the first period, the decision
maker observes a historical sample path $J'_1,\ldots,J'_T$, whose entries are
independent with $\Prob\{J'_t=j\}=p_{t,j}$. The
historical path need not be independent of the current arrivals, and the
expected regret is also taken over the historical path. We use the
proportions of the types among the historical requests of periods
$t,\ldots,T$ as the forecast,
\begin{equation}\label{eq:history-forecast}
    \eta_t=\frac1{s_t}\sum_{\tau=t}^T\mathbf e_{J'_\tau},
\end{equation}
where $\mathbf e_j$ denotes the $j$th unit vector in $\R^m$. Some components of
this forecast can be zero. This is the single-sample model of
\citet{balseiro_kumar_mirrokni_sivan_wang_2023}, \citet{ghuge_singla_wang_2025},
and \citet{feng_jiang_wang_2026}.

\medskip\noindent\emph{Model 4: Markov-modulated arrivals.}
Let $S_1,\ldots,S_T$ be a Markov chain on a finite state space $\cS$, with a
fixed irreducible transition matrix and an arbitrary initial distribution, and
suppose that, given $S_1,\ldots,S_T$, the types $J_1,\ldots,J_T$ are
independent with $\Prob\{J_t=j\mid S_1,\ldots,S_T\}=p_j(S_t)$. The Markov chain does not
depend on the decisions. Let $\varpi$ be the stationary distribution of the
chain, and suppose that the stationary type frequencies
$\bar p_j:=\sum_{s\in\cS}\varpi_sp_j(s)$ are positive. The decision maker does
not know the transition matrix nor the probabilities $p_j(s)$. The decision
maker also does not observe the states and has no historical data. We do not
assume that the chain is aperiodic. Model 1 is the special case with a single
state. In this model, the forecast $\eta_t$ we use is $\pi_t$, defined in
\eqref{eq:pi-estimate}. Markov-modulated
demands are studied by \citet{li_rusmevichientong_topaloglu_zhang_2025} and
\citet{jiang_2025}, and constant regret under Markovian arrivals with a known
distribution is established by \citet{vera_banerjee_2021} and
\citet[Section~3.3]{li_wang_zhang_2024}.

\begin{proposition}\label{prop:arrival-models}
The forecasts above satisfy Assumptions~\ref{ass:forecast-change}
and~\ref{ass:forecast-accuracy} with the following constants:
\begin{enumerate}[label=\textup{(\roman*)},leftmargin=2.2em]
\item in Model 1, with the forecast $\pi_t$, $D=2$ and $p_0=p_{\min}/2$;
\item in Model 2, with the forecast $\bar p_t$, $D=2\sqrt2$ and $p_0=p_{\min}$;
\item in Model 3, with the forecast \eqref{eq:history-forecast}, $D=2\sqrt2$
and $p_0=p_{\min}/2$;
\item in Model 4, with the forecast $\pi_t$, $D=2$ and
$p_0=\min_{j\in[m]}\bar p_j/2$.
\end{enumerate}
In each case, $\Lambda(\rho)$ depends on $\rho$, $m$, and the arrival model,
but not on $T$.
\end{proposition}

Assumption~\ref{ass:forecast-change} holds because each forecast is an average
over the past requests or over the remaining periods, and
Assumption~\ref{ass:forecast-accuracy} follows from concentration bounds; the
proof is given in Appendix~\ref{app:arrival-models}. Parts of the proposition
are known. Part~(ii) is essentially the all-time deviation property of
\citet[Definition~5]{vera_banerjee_2021}, which they verify for i.i.d.\
arrivals and note extends to time-varying arrival probabilities. In Model~3,
\citet{feng_jiang_wang_2026} use the same forecast and bound its error by
Hoeffding's inequality, assuming that the historical sample path is
independent of the current arrivals; our bound does not need this
independence. For Model~4, \citet[Example~6]{vera_banerjee_2021} and
\citet[Proposition~8]{li_wang_zhang_2024} consider Markovian arrivals with a
known aperiodic chain. For Model~4, we use the Hoeffding inequality of
\citet{moulos_2020} for irreducible finite-state Markov chains, which holds
for every initial distribution and does not require aperiodicity.

\section{Algorithm and main result}\label{sec:algorithm}

To develop a first-order policy, we relax the resource capacity constraints of
the fluid LP and apply a quadratic penalty to resource overuse, with the
forecast as the demand. In each period, the policy performs one projected gradient ascent step on the
coordinates of one type and selects, for the arriving request, the action with
the largest coordinate of its type.

\subsection{The policy}\label{subsec:policy}

Throughout the rest of the paper, $\one$ denotes a vector of ones of the appropriate dimension,
$\pos v$ denotes the componentwise positive part of a vector $v$, and
$\norm A_2$ denotes the spectral norm of a matrix $A$. For $y\in\R^K$, write
$y_j=(y_{ja})_{a\in\mathcal A_j}$ for the block of $y$ associated with type
$j$. Let $W_j\in\R_+^{d\times K_j}$ be the matrix with columns $w_{ja}$,
$a\in\mathcal A_j$, let $r_j:=(r_{ja})_{a\in\mathcal A_j}$, and let
$W:=(W_1,\ldots,W_m)\in\R_+^{d\times K}$, so that
\begin{equation*}
    Wy=\sum_{j=1}^m\sum_{a\in\mathcal A_j}w_{ja}y_{ja}
\end{equation*}
and $r=(r_1,\ldots,r_m)$. For $\eta\in\R_+^m$, define
\begin{equation}\label{eq:reduced-domain}
    \mathcal Y(\eta):=\prod_{j=1}^m\mathcal Y_j(\eta_j),
    \qquad
    \mathcal Y_j(\eta_j):=\bigl\{v\in\R_+^{K_j}:\one^Tv\le\eta_j\bigr\}.
\end{equation}
The set $\mathcal Y(\eta)$ consists of the allocations that satisfy the demand
constraints of the fluid LP, and $\cP(c,\eta)=\{y\in\mathcal Y(\eta):Wy\le c\}$.

For a capacity vector $c\in\R_+^d$, a demand vector $\eta\in\R_+^m$, and a
penalty parameter $\kappa>0$, define
\begin{equation}\label{eq:penalized-objective}
    f_c^\kappa(y):=r^Ty-\frac1{2\kappa}\norm{\pos{Wy-c}}_2^2,
    \qquad y\in\mathcal Y(\eta).
\end{equation}
We call the problem of maximizing $f_c^\kappa$ over $\mathcal Y(\eta)$ the
penalized problem. It replaces the capacity constraint of the
normalized fluid relaxation by a quadratic penalty on resource overuse, whose
coefficient $1/(2\kappa)$ increases as $\kappa$ decreases. The gradient of
$f_c^\kappa$ with respect to the block of type $j$ is
$r_j-W_j^T\pos{Wy-c}/\kappa$.

In each period, the policy takes one projected gradient ascent step on the
block of one type. Fix positive constants $L_1,\ldots,L_m$; the step size for
the block of type $j$ is $\kappa/L_j$. For $j\in[m]$ and $y\in\mathcal Y(\eta)$, let
$\PG_{c,\eta}^{\kappa,j}(y)$ be the vector obtained from $y$ by replacing its
block of type $j$ with
\begin{equation}\label{eq:observed-type-update}
    \bigl[\PG_{c,\eta}^{\kappa,j}(y)\bigr]_j
    :=\Pi_{\mathcal Y_j(\eta_j)}
    \left(
        y_j+\frac{\kappa r_j-W_j^T\pos{Wy-c}}{L_j}
    \right),
\end{equation}
where $\Pi_{\mathcal Y_j(\eta_j)}$ denotes the Euclidean projection onto
$\mathcal Y_j(\eta_j)$; the other blocks are unchanged. We refer to
$\kappa r_j/L_j$ and $-W_j^T\pos{Wy-c}/L_j$ as the reward term and the overuse
term of the update.

Theorem~\ref{thm:regret} below holds for every choice of the positive constants
$L_j$; we discuss a specific choice after Algorithm~\ref{alg:first-order-policy}.

In the online problem, the per-period capacity changes over time. In period
$t$, we set $c_t=C_t/s_t$, where $s_t$ is the number of remaining requests. The penalty parameter also changes over time. For a fixed exponent
$\gamma\in(0,1)$, we set
\begin{equation}\label{eq:penalty-schedule}
    \kappa_t=\ell_t^{-\gamma},
    \qquad
    \ell_t=\min\{t,s_t\}.
\end{equation}
The penalty parameter $\kappa_t$ is deterministic. It depends only on $t$, $T$,
and $\gamma$, and it is set before the request of period $t$ is observed. It
equals one in the first and last periods, decreases to its smallest value
$\lceil T/2\rceil^{-\gamma}$ in the middle of the horizon, and then increases
toward the end. A smaller $\kappa_t$ imposes a stronger penalty and gives
smaller step sizes $\kappa_t/L_j$. The demand vector in
\eqref{eq:penalized-objective} also changes over time. In period $t$, the
policy uses the forecast $\eta_t$ of Section~\ref{subsec:forecasts}.

In period $t$, the policy updates the block of one type, which we denote by
$I_t$. We require that, given the information $\mathcal F_{t-1}$ available
before period $t$, the type $I_t$ have a fixed distribution
$\vartheta=(\vartheta_1,\ldots,\vartheta_m)$ with positive entries:
\begin{equation}\label{eq:block-distribution}
    \Prob\bigl(I_t=j\mid\mathcal F_{t-1}\bigr)=\vartheta_j,
    \qquad j\in[m],\ t\in[T].
\end{equation}
Let $\vartheta_{\min}:=\min_{j\in[m]}\vartheta_j>0$. In Model 1 and in Model 2
with $p_t=p$, the type $J_t$ is independent of $\mathcal F_{t-1}$ and has
distribution $p$. The policy can then update the block of the arriving type,
$I_t=J_t$, and \eqref{eq:block-distribution} holds with $\vartheta=p$ without
the policy knowing $p$. Otherwise, the policy draws $I_t$ from a fixed
distribution with positive entries, for example the uniform distribution on
$[m]$, independently of the arrivals and of $\mathcal F_{t-1}$.

The policy stores, for each type $j$, the fractions $q_{t,j}\in\mathcal Y_j(1)$
of the demand of type $j$ that its current solution assigns to the actions in
$\mathcal A_j$; the remaining fraction $1-\one^Tq_{t,j}$ is assigned to
rejection. Initially, the demand of every type is assigned to rejection, that
is, $q_{1,j}=0$ for every $j$.
Because the forecast changes from one period to the next, the stored solution
is rescaled to the current forecast before the update. In period $t$, the
policy forms the allocation $y_t^0\in\mathcal Y(\eta_t)$ with blocks
\begin{equation*}
    y_{t,j}^0=\eta_{t,j}q_{t,j},\qquad j\in[m].
\end{equation*}
After observing $J_t=j$, it chooses the type $I_t=i$ and updates only the block
of type $i$, using the penalty parameter $\kappa_t$:
\begin{equation*}
    \widehat y_t=\PG_{c_t,\eta_t}^{\kappa_t,i}(y_t^0).
\end{equation*}
It sets $\widehat q_{t,i}=\widehat y_{t,i}/\eta_{t,i}$ if $\eta_{t,i}>0$ and
$\widehat q_{t,i'}=q_{t,i'}$ for every other type $i'$. If $\eta_{t,i}=0$, the
block of type $i$ is zero both before and after the update, and the policy
keeps $\widehat q_{t,i}=q_{t,i}$. The policy then selects an action
$a\in\mathcal A_j\cup\{0\}$ with the largest fraction $\widehat q_{t,ja}$, where
$\widehat q_{t,j0}:=1-\one^T\widehat q_{t,j}$ is the fraction of rejection. The selected action
is taken if it fits within the remaining capacity; otherwise, the request is
rejected. The update is performed whether or not the selected action fits.
Finally, $q_{t+1,i'}=\widehat q_{t,i'}$ for every $i'\in[m]$, so that
$y_{t+1,i'}^0=\eta_{t+1,i'}\widehat q_{t,i'}$. Thus the fractions of each type
are carried over, while the demand of the type is updated to the next
forecast.

\begin{algorithm}[H]
\caption{Primal first-order policy}\label{alg:first-order-policy}
\begin{algorithmic}[1]
\Require Horizon $T$; initial capacity $B_T$; action menus, rewards, and
consumptions; exponent $\gamma\in(0,1)$; positive constants $L_1,\ldots,L_m$;
a forecast $\eta_1,\ldots,\eta_T$; a rule for choosing $I_t$ that satisfies
\eqref{eq:block-distribution}.
\State Set $C_1\gets B_T$ and $q_{1,j}\gets0$ for every $j\in[m]$.
\For{$t=1,2,\ldots,T$}
    \State Set $s_t\gets T-t+1$, $\ell_t\gets\min\{t,s_t\}$, and
    $\kappa_t\gets\ell_t^{-\gamma}$.
    \State Compute $\eta_t$, and set $c_t\gets C_t/s_t$ and
    $y_{t,j}^{0}\gets\eta_{t,j}q_{t,j}$ for every $j\in[m]$.
    \State Observe the current type $j\gets J_t$, and choose the type $i\gets I_t$.
    \State Set $\widehat y_t\gets\PG_{c_t,\eta_t}^{\kappa_t,i}(y_t^0)$.
    \State Set $\widehat q_{t,i}\gets\widehat y_{t,i}/\eta_{t,i}$ if
    $\eta_{t,i}>0$ and $\widehat q_{t,i}\gets q_{t,i}$ otherwise, and set
    $\widehat q_{t,i'}\gets q_{t,i'}$ for every $i'\ne i$.
    \State Set $\widehat q_{t,j0}\gets1-\one^T\widehat q_{t,j}$ and select
    as $a_t^*$ the smallest index in
    $\argmax_{a\in\mathcal A_j\cup\{0\}}\widehat q_{t,ja}$.
    \If{$w_{ja_t^*}\le C_t$ coordinatewise}
        \State Set $U_t\gets a_t^*$.
    \Else
        \State Set $U_t\gets0$.
    \EndIf
    \State Set $C_{t+1}\gets C_t-w_{jU_t}$ and $q_{t+1,i'}\gets\widehat q_{t,i'}$
    for every $i'\in[m]$.
\EndFor
\end{algorithmic}
\end{algorithm}

In period $t$, the update involves only the coordinates of type $I_t$ and the
vector $Wy_t^0\in\R^d$. If $I_t=i$, then
\begin{equation*}
    W\widehat y_t
    =Wy_t^0+W_i\bigl(\widehat y_{t,i}-y_{t,i}^0\bigr).
\end{equation*}
For the forecasts of Models 1, 3, and 4, the vector $Wy_{t+1}^0$ can also be
obtained without computing the coordinates of the other types. For the
forecast $\pi_t$ in \eqref{eq:pi-estimate},
\begin{equation*}
    Wy_{t+1}^0
    =\left(1-\frac1{t+m}\right)W\widehat y_t
      +\frac1{t+m}W_j\widehat q_{t,j},
    \qquad j=J_t,
\end{equation*}
and, for the forecast \eqref{eq:history-forecast} and $t<T$,
\begin{equation*}
    Wy_{t+1}^0
    =\frac{s_tW\widehat y_t-W_{j'}\widehat q_{t,j'}}{s_t-1},
    \qquad j'=J'_t.
\end{equation*}
For the forecast $\bar p_t$ of Model 2 with identically distributed arrivals,
$\bar p_t=p$ and $Wy_{t+1}^0=W\widehat y_t$. With time-varying $p_t$, the
forecast of every type changes, and computing $Wy_{t+1}^0$
requires the vectors $W_j\widehat q_{t,j}$ of all types.


The update \eqref{eq:observed-type-update} is a step of randomized
block-coordinate ascent for the penalized problem, whose blocks are drawn from the fixed distribution in \eqref{eq:block-distribution}: the new block of type $j$ maximizes, over $\mathcal Y_j(\eta_j)$, the
linearization of $f_c^\kappa$ at $y$ minus $L_j/(2\kappa)$ times the squared
Euclidean distance to $y_j$. \citet{nesterov_2012} analyzes such methods for
constrained smooth convex problems over block-separable sets with uniform sampling, and
\citet{richtarik_takac_2014} and \citet{salzo_villa_2022} extend the analysis
to composite problems and nonuniform sampling. 

\subsection{Main result}\label{subsec:main-result}

\begin{theorem}\label{thm:regret}
Suppose that the forecast satisfies Assumptions~\ref{ass:forecast-change}
and~\ref{ass:forecast-accuracy} and that the types $I_t$ satisfy
\eqref{eq:block-distribution}. Fix an exponent $\gamma\in(0,1)$ and positive
constants $L_1,\ldots,L_m$. There is a finite constant $C$, independent of
$T$, such that the policy in Algorithm~\ref{alg:first-order-policy}, denoted by
$\mu$, satisfies
\begin{equation*}
    \Reg_T(\mu)\le C,
    \qquad T\ge 1.
\end{equation*}
\end{theorem}

Algorithm~\ref{alg:first-order-policy} uses the horizon $T$, which also
determines the per-period capacity $c_t$, together with the initial capacity,
the action menus, the rewards, the resource-consumption vectors, the observed
requests, the forecast, the rule for choosing $I_t$, the exponent $\gamma$,
and the constants $L_j$.
It does not use the constants $D$, $p_0$, and
$\Lambda(\rho)$ of Assumptions~\ref{ass:forecast-change}
and~\ref{ass:forecast-accuracy} or bounds on the sensitivity constants of
Section~\ref{sec:preliminary-analysis}, and it does not solve a linear program,
even for initialization. These quantities enter only the proof. When
$I_t=J_t$, the distribution $\vartheta$ consists of the arrival probabilities,
and the policy does not need them. The constant $C$ obtained in the proof depends on $\gamma$, the
constants $L_j$, the action menus, the resource-consumption matrix $W$, the
reward vector $r$, $1/\vartheta_{\min}$, $D$, $1/p_0$, and the value
$\Lambda(\rho)$ at a number $\rho>0$ determined by $p_0$, $W$, and the action
menus; it need not remain bounded as $\gamma$ approaches $0$ or $1$.

The notation $O(1)$ refers only to the dependence on the horizon, and the
constant is not dimension-free. In fact, \citet{zhang_exponential_2026}
constructs finite-type NRM problems for which every
online policy has regret exponential in the number of resources at a suitably
chosen horizon. Thus, in the worst case, our bound can also be exponential in
the number of resources. Similarly, \citet[Lemma~1]{arlotto_gurvich_2019} give
a multisecretary problem with i.i.d.\ arrivals from a known distribution whose
smallest arrival probability is $p_{\min}$, for which every online policy has
regret at least a constant times $1/p_{\min}$ at a suitably chosen horizon.
Thus, in the worst case, our bound is also at least of order $1/p_{\min}$.

Theorem~\ref{thm:regret} and Proposition~\ref{prop:arrival-models} give the
following regret bounds for the arrival models of
Section~\ref{subsec:arrival-models}.

\begin{corollary}\label{cor:arrival-models}
In each of Models 1--4, let Algorithm~\ref{alg:first-order-policy} use the
forecast given in Section~\ref{subsec:arrival-models}. Let $I_t=J_t$ in Model 1
and in Model 2 with $p_t=p$ for every $t$. Otherwise, let $I_t$ be drawn from
the uniform distribution on $[m]$, independently of the arrivals, of the
historical sample path, and of the past choices. Then, for every $\gamma\in(0,1)$ and all positive constants
$L_1,\ldots,L_m$, there is a finite constant $C$, independent of $T$, such that
$\Reg_T(\mu)\le C$ for $T\ge1$.
\end{corollary}

\begin{proof}
By Proposition~\ref{prop:arrival-models}, the forecast satisfies
Assumptions~\ref{ass:forecast-change} and~\ref{ass:forecast-accuracy}. If
$I_t=J_t$, then $J_t$ is independent of $\mathcal F_{t-1}$ and has distribution
$p$, so \eqref{eq:block-distribution} holds with $\vartheta=p$. Otherwise,
$I_t$ is independent of $\mathcal F_{t-1}$, so \eqref{eq:block-distribution}
holds with $\vartheta_j=1/m$. Theorem~\ref{thm:regret} therefore applies.
\end{proof}

\subsection{Overview of the analysis}\label{subsec:analysis-overview}

The proof of Theorem~\ref{thm:regret} consists of three steps. First, following
\citet{vera_banerjee_2021}, we bound the regret by the expected sum of the
one-period losses (Lemma~\ref{lem:loss-telescoping}). Second, we show that the
one-period loss is zero if the forecast is accurate in the sense of
Assumption~\ref{ass:forecast-accuracy} and the solution $\widehat y_t$
obtained by the projected gradient update \eqref{eq:observed-type-update} is
within a fixed distance $\varepsilon$ of the optimal solution set
$\mathcal O(c_t,\eta_t)$ of the current fluid LP (Proposition~\ref{prop:loss-gap}).
This is because the selected action carries at least a fraction $1/(K_j+1)$ of the
demand of the arriving type $j$, so that, by LP sensitivity, some optimal solution
of the hindsight LP also takes this action. Thus, the solutions $\widehat y_t$
need only stay within a fixed distance of the optimal set, and they do not need
to converge. Third, we show that a high moment of this distance has a bounded
sum over the horizon (Lemma~\ref{lem:period-gap-bound}), which implies, by
Markov's inequality, that the probabilities that the distance exceeds
$\varepsilon$ also have a bounded sum. Since the probabilities that the forecast
is not accurate also have a bounded sum by
Assumption~\ref{ass:forecast-accuracy}, these steps prove
Theorem~\ref{thm:regret}.

The third step is the main step of the analysis. Randomized block-coordinate
methods for a fixed problem converge linearly when the objective satisfies an
error bound \citep{necoara_clipici_2016,salzo_villa_2022}, but these results do
not apply directly here, for two reasons. First, the projected gradient update
uses the gradient of the penalized objective, whereas the error bound that we
use concerns the optimal set of the fluid LP. Second, the optimal set of the
fluid LP moves in every period, whereas the policy updates only one block per
period, with a step size $\kappa_t/L_j$ that decreases from $1$ in the first
period to order $T^{-\gamma}$ in the middle of the horizon and increases again
toward the end. We therefore track the distance to the optimal set directly.
Let $e_t$ be the distance, in
the norm \eqref{eq:weighted-norm}, from the solution $y_t^0$ before the update
in period $t$ to $\mathcal O(c_t,\eta_t)$, and let
$\zeta_t:=\norm{\pos{Wy_t^0-c_t}}_1$ be the planned resource overuse of
$y_t^0$. In each period, the distance changes twice. The
projected gradient update first moves the solution, and then the optimal set
moves from $\mathcal O(c_t,\eta_t)$ to $\mathcal O(c_{t+1},\eta_{t+1})$. Since
the block of type $j$ is updated with probability $\vartheta_j$, we use the
weights $L_j/\vartheta_j$ in \eqref{eq:weighted-norm}, so that the expected
change in the squared distance to a point of $\mathcal O(c_t,\eta_t)$ equals
the change produced by updating all blocks at once, measured with the weights
$L_j$ \citep{richtarik_takac_2014,salzo_villa_2022}. Expanding this change and
adding the movement of the optimal set, we obtain
\[
    \E\bigl[e_{t+1}^2\mid\mathcal F_{t-1}\bigr]
    \le e_t^2
      -2\kappa_t\bigl(V_r(c_t,\eta_t)-r^Ty_t^0\bigr)
      +C\bigl(\kappa_t^2+\zeta_t^2\bigr)
      +C\Bigl(\frac{e_t}{\ell_t}+\frac{1}{\ell_t^2}\Bigr)
\]
for a constant $C$, where $\mathcal F_{t-1}$ is the information available
before period $t$. The second term on the right-hand side comes from the reward
term $\kappa_tr_j/L_j$ in \eqref{eq:observed-type-update}, and the third term
comes from the squared length of the update step. The overuse term
$-W_j^T\pos{Wy_t^0-c_t}/L_j$ in \eqref{eq:observed-type-update} does not appear
in the inequality. Since every optimal solution satisfies the capacity
constraint, the contribution of this term is nonpositive. The last term comes from the
movement of the optimal set. Between periods, $\eta_t$ changes by
$O(1/\ell_t)$ by Assumption~\ref{ass:forecast-change}, and the per-period
capacity $c_t$ changes by $O(1/s_t)$ after each of its components is capped at
the largest consumption of that resource by a single action, which does not
change the optimal set. By LP sensitivity, the optimal set therefore moves by
$O(1/\ell_t)$ with $\ell_t=\min\{t,s_t\}$ (Lemma~\ref{lem:target-drift}).

We prove in Lemma~\ref{lem:reward-feasibility} that the reward shortfall
$V_r(c_t,\eta_t)-r^Ty_t^0$ is at least $\sigma e_t-H_4\zeta_t$ for constants
$\sigma>0$ and $H_4$. Hence, the second term is at most
$-2\sigma\kappa_te_t+2H_4\kappa_t\zeta_t$, so the reward term decreases the
squared distance by an amount proportional to $\kappa_te_t$. We then prove in
Lemma~\ref{lem:violation-moments} that $\E\zeta_t^{2k}\le V_k\kappa_t^{2k}$ for
every $k\ge1$, so the third term and the term $2H_4\kappa_t\zeta_t$ are of order
$\kappa_t^2$ in expectation. The planned overuse is of order $\kappa_t$ because
the overuse term of the projected gradient update removes overuse at a geometric
rate, whereas the reward term and the changes of $c_t$ and $\eta_t$ create new
overuse at a rate of order $\kappa_t$. Consequently, at a fixed distance from
the optimal set, the decrease of order $\kappa_te_t$ dominates the terms of
order $\kappa_t^2$ when $\kappa_t\to0$, and it dominates the movement of the
optimal set, of order $e_t/\ell_t$, when $\ell_t\kappa_t\to\infty$. The choice
\eqref{eq:penalty-schedule}, $\kappa_t=\ell_t^{-\gamma}$ with $\gamma\in(0,1)$,
satisfies both requirements.

These bounds by themselves do not give a summable bound on $\E e_t^2$. The
resulting bound on $\E e_t^2$ is of order $\kappa_t+(\ell_t\kappa_t)^{-2}$, which
is at least $\ell_t^{-2/3}$ for every choice of $\kappa_t$, so its sum still grows with $T$.
However, because the decision requires only a fixed accuracy $\varepsilon$, we
can apply Markov's inequality to the $2k$th power of the distance for a fixed
integer $k\ge2$. We show in Lemmas~\ref{lem:moment-contraction}
and~\ref{lem:period-gap-bound} that the $2k$th moments have a bounded sum when
$k>\max\{1/\gamma,1/(2(1-\gamma))\}$.

\section{Properties of the fluid LP}\label{sec:preliminary-analysis}

The regret analysis in Section~\ref{sec:regret-analysis} compares the iterates
of the policy with the optimal solutions of the fluid LP, which requires LP
sensitivity bounds and an LP error bound.

\subsection{Sensitivity results}\label{subsec:sensitivity}

In the regret analysis, the per-period capacity $c_t$ and the forecast $\eta_t$
change from one period to the next, so we need to compare problems with
different parameters.

For $x\in\R^K$ and a nonempty set $D\subseteq\R^K$, define
\begin{equation*}
    \dist_2(x,D):=\inf_{z\in D}\norm{x-z}_2.
\end{equation*}
For two nonempty compact sets $D_1,D_2\subseteq\R^K$, define their Euclidean
Hausdorff distance by
\begin{equation*}
    d_H(D_1,D_2)
    :=\max\left\{
        \sup_{x\in D_1}\dist_2(x,D_2),
        \sup_{z\in D_2}\dist_2(z,D_1)
    \right\}.
\end{equation*}
Thus, $d_H(D_1,D_2)\le\varepsilon$ means that every element of either set is
within Euclidean distance $\varepsilon$ of an element of the other set.

\begin{lemma}
\label{lem:optimal-face-sensitivity}
There is a finite constant $H_1\ge1$, depending only on $W$ and the action
menus, such that, for $c,c'\in\R_+^d$, $\eta,\eta'\in\R_+^m$, and every
objective vector $\widetilde r\in\R^K$,
\begin{equation*}
d_H\left(
    \argmax_{x\in\cP(c,\eta)}\widetilde r^Tx,
    \argmax_{x\in\cP(c',\eta')}\widetilde r^Tx
\right)
\le
H_1\sqrt{\norm{c-c'}_2^2+\norm{\eta-\eta'}_2^2}.
\end{equation*}
\end{lemma}

Applying Lemma~\ref{lem:optimal-face-sensitivity} with $\widetilde r=r$ to an
optimal solution of either LP and using the Cauchy--Schwarz inequality, we also
obtain
\begin{equation}\label{eq:fluid-value-lipschitz}
    \bigl|V_r(c,\eta)-V_r(c',\eta')\bigr|
    \le H_1\norm r_2\sqrt{\norm{c-c'}_2^2+\norm{\eta-\eta'}_2^2}.
\end{equation}

\begin{lemma}
\label{lem:reward-level-sensitivity}
There is a finite constant $H_3\ge1$, depending only on $W$, $r$, and the
action menus, such that, for $c\in\R_+^d$, $\eta\in\R_+^m$, and
$z\in\cP(c,\eta)$,
\begin{equation}\label{eq:optimal-face-error-bound}
\dist_2\left(
    z,
    \argmax_{x\in\cP(c,\eta)}r^Tx
\right)
\le
H_3\bigl[V_r(c,\eta)-r^Tz\bigr].
\end{equation}
\end{lemma}

Lemmas~\ref{lem:optimal-face-sensitivity}
and~\ref{lem:reward-level-sensitivity} follow from Theorems~2.4 and~2.2,
respectively, of \citet{mangasarian_shiau_1987}. Their Theorem~2.4 states that,
when the right-hand sides of a linear program change, every optimal solution
has an optimal solution of the new program within $\ell_\infty$-distance at
most a constant times the change in the right-hand sides, where the constant
depends only on the constraint matrices, not on the objective.
Measuring the change in the Euclidean norm, applying the bound in both
directions, and using $\norm{\cdot}_2\le\sqrt K\norm{\cdot}_\infty$ gives
Lemma~\ref{lem:optimal-face-sensitivity}. Related
sensitivity bounds have also been used by \citet{vera_banerjee_2021} and
\citet{jiang_ma_zhang_2025} in their analyses of online algorithms without
nondegeneracy assumptions. Theorem~2.2 of \citet{mangasarian_shiau_1987} is
the corresponding statement for systems of linear inequalities: if two systems with the same
constraint matrix but different right-hand sides both have solutions, then
every solution of one system is within $\ell_\infty$-distance at most a
constant times the change in the right-hand sides of a solution of the other.
To obtain
Lemma~\ref{lem:reward-level-sensitivity} from it, note that the set $\{x\in\cP(c,\eta):-r^Tx\le u\}$ is nonempty when
$u\ge-V_r(c,\eta)$. For $z\in\cP(c,\eta)$, it contains $z$ when $u=-r^Tz$, and
it equals $\argmax_{x\in\cP(c,\eta)}r^Tx$ when $u=-V_r(c,\eta)$. These two
right-hand sides differ by $V_r(c,\eta)-r^Tz$, so their Theorem~2.2 and the
same norm conversion give \eqref{eq:optimal-face-error-bound}. The constant
$H_3$ depends on $r$ because $r$ enters the constraint matrix. In the
terminology of \citet{burke_ferris_1993},
Lemma~\ref{lem:reward-level-sensitivity} states that the optimal solution set
of the fluid LP is a set of weak sharp minima of $-r^Tx$ over $\cP(c,\eta)$;
they show that the optimal solution set of every linear program has this
property.

We note that neither Lemma~\ref{lem:optimal-face-sensitivity} nor
Lemma~\ref{lem:reward-level-sensitivity} requires nondegeneracy. They compare
entire solution sets.

\subsection{An error bound}\label{subsec:reduced-error-bound}

For $c\in\R_+^d$ and $\eta\in\R_+^m$, let
\begin{equation}\label{eq:reduced-optimal-set}
    \mathcal O(c,\eta)
    :=\argmax_{x\in\cP(c,\eta)}r^Tx
\end{equation}
denote the optimal solution set of the fluid LP. It is nonempty and compact,
because $\cP(c,\eta)$ contains the zero vector and is compact.

The policy updates the block of type $j$ with probability $\vartheta_j$ and step size
proportional to $1/L_j$. In the regret analysis, we measure distances in the norm
\begin{equation}\label{eq:weighted-norm}
    \norm{z}_{L/\vartheta}^2:=\sum_{j=1}^m\frac{L_j}{\vartheta_j}\norm{z_j}_2^2,
    \qquad
    \dist_{L/\vartheta}(y,D):=\inf_{z\in D}\norm{y-z}_{L/\vartheta},
\end{equation}
for $y,z\in\R^K$ and a nonempty set $D\subseteq\R^K$. When $I_t=J_t$, the
weights of this norm involve the unknown arrival probabilities.
Let $L_{\min}:=\min_jL_j$, $L_{\max}:=\max_jL_j$, $K_{\max}:=\max_jK_j$, and
$M_1:=\sqrt{L_{\max}/\vartheta_{\min}}$. Since $\vartheta_{\min}\le\vartheta_j\le1$, every $z\in\R^K$ satisfies
\begin{equation}\label{eq:norm-equivalence}
    \sqrt{L_{\min}}\,\norm z_2\le\norm z_{L/\vartheta}\le M_1\norm z_2.
\end{equation}
Finally, if $\eta$ is a probability vector, then every
$y\in\mathcal Y(\eta)$ has nonnegative coordinates that sum to at most one, so
$\norm y_2\le1$, and
\begin{equation}\label{eq:reduced-diameter}
    \norm{y-z}_{L/\vartheta}\le 2M_1,
    \qquad y,z\in\mathcal Y(\eta).
\end{equation}

The penalized problem allows the iterates of the policy to violate the capacity
constraint of the fluid LP, so we need an error bound that applies to allocations that overuse
resources. The next lemma combines Lemmas~\ref{lem:optimal-face-sensitivity}
and~\ref{lem:reward-level-sensitivity} into such a bound. It shows that the
reward shortfall of an allocation relative to the fluid LP is at least
proportional to its distance to the optimal solution set, up to a term
proportional to its overuse.

\begin{lemma}\label{lem:reward-feasibility}
Let $c\in\R_+^d$, let $\eta\in\R_+^m$, and let $y\in\mathcal Y(\eta)$. Then
\begin{equation}\label{eq:reward-feasibility}
    V_r(c,\eta)-r^Ty
    \ge\sigma\dist_{L/\vartheta}\bigl(y,\mathcal O(c,\eta)\bigr)
       -H_4\norm{\pos{Wy-c}}_1,
\end{equation}
where $\sigma:=1/(M_1H_3)$ and $H_4:=H_1(H_3^{-1}+\norm r_2)$.
\end{lemma}

\begin{proof}
Let $v:=\pos{Wy-c}$. Since $Wy\le c+v$, the allocation $y$ belongs to
$\cP(c+v,\eta)$. Lemma~\ref{lem:reward-level-sensitivity}, applied with
capacity $c+v$, gives
\begin{equation*}
    \dist_2\bigl(y,\mathcal O(c+v,\eta)\bigr)
    \le H_3\bigl[V_r(c+v,\eta)-r^Ty\bigr].
\end{equation*}
By Lemma~\ref{lem:optimal-face-sensitivity}, every optimal solution of the
fluid LP with capacity $c+v$ is within Euclidean distance $H_1\norm v_2$ of an
optimal solution of the fluid LP with capacity $c$, and
\eqref{eq:fluid-value-lipschitz} gives
$V_r(c+v,\eta)\le V_r(c,\eta)+H_1\norm r_2\norm v_2$. Therefore,
\begin{equation*}
    \dist_2\bigl(y,\mathcal O(c,\eta)\bigr)
    \le H_3\bigl[V_r(c,\eta)-r^Ty\bigr]
       +H_1\bigl(1+H_3\norm r_2\bigr)\norm v_2.
\end{equation*}
By \eqref{eq:norm-equivalence}, $\dist_{L/\vartheta}(y,\mathcal O(c,\eta))$ is
at most $M_1$ times the left side. Dividing by $M_1H_3$ and using
$\norm v_2\le\norm v_1$ proves \eqref{eq:reward-feasibility}.
\end{proof}

\section{Regret analysis}\label{sec:regret-analysis}

We now prove Theorem~\ref{thm:regret}, following the outline in
Section~\ref{subsec:analysis-overview}.

\subsection{\texorpdfstring{{One-period loss}}{One-period loss}}\label{subsec:one-period-loss}

For an arrival sample path $J_1,\ldots,J_T$, let $N_{t,j}$
be the number of type $j$ requests among $J_t,\ldots,J_T$, and write
$N_t=(N_{t,1},\ldots,N_{t,m})$. For a feasible policy $\mu$, define its
one-period loss in period $t$ by
\begin{equation}\label{eq:one-period-loss}
\begin{aligned}
    R_t^{\mu}=\Phi(C_t,N_t)-r_{J_t U_t}
    -\Phi(C_t-w_{J_tU_t},N_{t+1}).
\end{aligned}
\end{equation}
This is the decrease in the LP value after fixing the current action $U_t$, net
of the reward earned in period $t$. As in the compensated coupling of
\citet{vera_banerjee_2021}, the one-period losses give the following
sample-path bound.

\begin{lemma}\label{lem:loss-telescoping}
Let $\mu$ be any feasible policy.  For every sample path of arrivals, we have
\begin{equation*}
    \OPT_T-\REV_T^\mu \le\sum_{t=1}^{T}R_t^{\mu}.
\end{equation*}
\end{lemma}

\begin{proof}
    Since $C_{t+1}=C_t-w_{J_tU_t}$, the sum of the one-period losses telescopes,
    \[
    \sum_{t=1}^{T}R_t^{\mu}
    =\Phi(C_1,N_1)-\sum_{t=1}^{T}r_{J_tU_t}-\Phi(C_{T+1},N_{T+1})
    =\Phi(C_1,N_1)-\REV_T^\mu
    \]
    because $N_{T+1}=0$ and $\Phi(C_{T+1},0)=0$. The allocation attaining the
    hindsight optimum is feasible for the LP defining $\Phi(C_1,N_1)$, so
    $\OPT_T\le\Phi(C_1,N_1)$, which implies the desired inequality.
\end{proof}

Taking expectations in Lemma~\ref{lem:loss-telescoping} reduces the proof of
Theorem~\ref{thm:regret} to bounding $\sum_{t=1}^T\E R_t^\mu$. We next
characterize the actions that cause no one-period loss. The feasible set of the
LP defining $\Phi(C,N)$ is $\cP(C,N)$. For $z\in\cP(C,N)$, let
$z_{j0}:=N_j-\sum_{b\in\mathcal A_j}z_{jb}$ be the number of type $j$ requests
that $z$ rejects. Recall that $\mathbf e_j$ denotes the
$j$th unit vector in $\R^m$, and let $\mathbf e_{ja}$ denote the coordinate
vector in $\R^K$ indexed by $(j,a)$ for $a\in\mathcal A_j$, with
$\mathbf e_{j0}:=0$. Suppose that $N_j\ge1$, and let $a\in\mathcal A_j\cup\{0\}$
satisfy $C-w_{ja}\ge0$. Assigning one type $j$ request to action $a$ reduces
the capacity by $w_{ja}$ and the demand of type $j$ by one: the map
$z\mapsto z-\mathbf e_{ja}$ is a bijection from $\{z\in\cP(C,N):z_{ja}\ge1\}$
onto $\cP(C-w_{ja},N-\mathbf e_j)$. Hence,
\begin{equation}\label{eq:fixed-action-value}
    r_{ja}+\Phi(C-w_{ja},N-\mathbf e_j)
    =\max\bigl\{r^Tz:z\in\cP(C,N),\ z_{ja}\ge1\bigr\},
    \qquad a\in\mathcal A_j\cup\{0\}.
\end{equation}
Thus, $\Phi(C,N)-r_{ja}-\Phi(C-w_{ja},N-\mathbf e_j)$ is
the decrease in the LP value caused by requiring at least one type $j$ request
to be assigned to action $a$. The first part
of the next lemma identifies when this requirement
has no effect on the LP value. The second part bounds the decrease
independently of $C$ and $N$.

\begin{lemma}\label{lem:one-period-loss-properties}
Let $C\in\R_+^d$ and $N\in\R_+^m$, and let $j\in[m]$ and
$a\in\mathcal A_j\cup\{0\}$ satisfy $N_j\ge1$ and $C-w_{ja}\ge0$. Then:
\begin{enumerate}[label=\textup{(\roman*)},leftmargin=2.2em]
\item $\Phi(C,N)-r_{ja}-\Phi(C-w_{ja},N-\mathbf e_j)=0$ if and only if some
optimal solution $z$ of the LP defining $\Phi(C,N)$ satisfies $z_{ja}\ge1$.
\item We have
\begin{equation*}
    0\le\Phi(C,N)-r_{ja}-\Phi(C-w_{ja},N-\mathbf e_j)\le H_2(1+\norm r_2),
\end{equation*}
where $w_{\max}:=\max_{j,a}\norm{w_{ja}}_2$ and
$H_2:=1+H_1\sqrt{w_{\max}^2+1}$.
\end{enumerate}
\end{lemma}

The proof is given in Appendix~\ref{app:one-period-loss-properties}. In period
$t$, the current request is counted in $N_t$, so $N_{t,J_t}\ge1$, and
$N_{t+1}=N_t-\mathbf e_{J_t}$. Since $\mu$ is feasible and every $w_{ja}$ is
nonnegative, $C_t-w_{J_tU_t}=C_{t+1}\ge B_T-\sum_{\tau=1}^Tw_{J_\tau U_\tau}\ge0$.
Therefore, Lemma~\ref{lem:one-period-loss-properties} applies to $R_t^\mu$ with
$C=C_t$, $N=N_t$, $j=J_t$, and $a=U_t$. In particular,
\begin{equation*}
    0\le R_t^\mu\le\bar R:=H_2(1+\norm r_2)
\end{equation*}
in every period.

Recall that $c_t=C_t/s_t$ and $\nu_t=N_t/s_t$. By the scaling in
Section~\ref{subsec:allocation-model}, $s_tx$ is an optimal solution of the LP
defining $\Phi(C_t,N_t)$ whenever $x$ maximizes $r^Tx$ over $\cP(c_t,\nu_t)$.

\subsection{Bounding the one-period loss by the distance to the optimal set}\label{subsec:loss-gap}

In this subsection, we bound $R_t^\mu$ in terms of the distance from
$\widehat y_t$ to $\mathcal O(c_t,\eta_t)$ when the forecast $\eta_t$ is
accurate and enough requests remain. For the
policy in Algorithm~\ref{alg:first-order-policy}, define
\begin{equation}\label{eq:distance-definitions}
    e_t:=\dist_{L/\vartheta}\bigl(y_t^0,\mathcal O(c_t,\eta_t)\bigr),
    \qquad
    \widehat e_t:=\dist_{L/\vartheta}\bigl(\widehat y_t,\mathcal O(c_t,\eta_t)\bigr).
\end{equation}
Thus, $e_t$ is the distance before the update, while $\widehat e_t$ is the
distance after the update. Both are measured to the optimal solution set of the
fluid LP with the current per-period capacity and forecast, and both are at
most $2M_1$ by \eqref{eq:reduced-diameter}.

Let $p_0$ be the constant in Assumption~\ref{ass:forecast-accuracy}. For every
period $t$, define
\begin{equation*}
    \mathcal G_t
    :=
    \left\{
        \norm{\eta_t-\nu_t}_2\le\rho,
        \quad
        \eta_{t,J_t}\ge p_0
    \right\},
    \qquad
    \rho
    :=
    \frac{p_0}
    {32(K_{\max}+1)^{3/2}(H_1+1)}.
\end{equation*}
By Assumption~\ref{ass:forecast-accuracy},
$\sum_{t=1}^T\Prob(\mathcal G_t^c)\le\Lambda(\rho)$. The bound
$\norm{\eta_t-\nu_t}_2\le\rho$ allows us to use
Lemma~\ref{lem:optimal-face-sensitivity} to compare optimal solutions with
demands $\eta_t$ and $\nu_t$. Let
\begin{equation*}
    \bar s
    :=\left\lceil
        \max\left\{2,\frac{16(K_{\max}+1)}{p_0}\right\}
    \right\rceil,
    \qquad
    \varepsilon:=\frac{p_0}{4(K_{\max}+1)M_2},
    \qquad
    M_2:=\sqrt{K_{\max}/L_{\min}}.
\end{equation*}
The condition $s_t>\bar s$ ensures that, when $\widehat y_t$ is within distance
$\varepsilon$ of $\mathcal O(c_t,\eta_t)$, an optimal solution of the fluid LP
with demand $\nu_t$ assigns more than $1/s_t$ to the selected action, or, if
the selected action is rejection, leaves more than $1/s_t$ of the demand of
the arriving type unassigned. The
tolerance $\varepsilon$ depends on $p_0$, the problem instance, and the
constants $L_j$, but not on $T$.

\begin{proposition}
\label{prop:loss-gap}
Consider the policy $\mu$ in Algorithm~\ref{alg:first-order-policy} and a
period $t$ with $s_t>\bar s$. On the event $\mathcal G_t$, if
$\widehat e_t\le\varepsilon$, then the selected action $a_t^*$ fits within the
remaining capacity and $R_t^\mu=0$. Consequently, for every integer $k\ge1$,
\begin{equation}\label{eq:loss-moment-bound}
    R_t^\mu
    \le
    \bar R\,\varepsilon^{-2k}\,\widehat e_t^{\,2k}
    +\bar R\,\one_{\mathcal G_t^c}.
\end{equation}
\end{proposition}

\begin{proof}
Fix a period $t$ as in the statement, assume that $\mathcal G_t$ occurs and
$\widehat e_t\le\varepsilon$, and write $j=J_t$ and $a=a_t^*$. The block of
type $j$ satisfies $\widehat y_{t,j}=\eta_{t,j}\widehat q_{t,j}$ whether or not
it was updated in period $t$. The action $a$ is selected from a largest of the
$K_j+1$ fractions of type $j$, including the fraction
$\widehat q_{t,j0}=1-\one^T\widehat q_{t,j}$ of rejection, and
$\eta_{t,j}\ge p_0$ on $\mathcal G_t$. Hence,
$\eta_{t,j}\widehat q_{t,ja}\ge\eta_{t,j}/(K_j+1)\ge p_0/(K_{\max}+1)$, that is,
\begin{equation}\label{eq:selected-action-lower-bound}
    \widehat y_{t,ja}\ge\frac{p_0}{K_{\max}+1}
    \quad\text{if }a\in\mathcal A_j,
    \qquad
    \eta_{t,j}-\one^T\widehat y_{t,j}\ge\frac{p_0}{K_{\max}+1}
    \quad\text{if }a=0.
\end{equation}
Choose $z\in\mathcal O(c_t,\eta_t)$ nearest to $\widehat y_t$ in the norm
$\norm\cdot_{L/\vartheta}$. Lemma~\ref{lem:optimal-face-sensitivity}, applied
with the same capacity $c_t$ and the demands $\eta_t$ and $\nu_t$, gives an
optimal solution $u$ of the fluid LP with capacity $c_t$ and demand $\nu_t$
such that $\norm{u-z}_2\le H_1\norm{\eta_t-\nu_t}_2\le H_1\rho$. By
\eqref{eq:norm-equivalence},
$\sqrt{K_j}\norm{\widehat y_{t,j}-z_j}_2\le M_2\widehat e_t\le M_2\varepsilon$.
Therefore,
\begin{align*}
    \delta
    &:=\norm{\widehat y_{t,j}-z_j}_1+\norm{z_j-u_j}_1+|\eta_{t,j}-\nu_{t,j}|\\
    &\le\sqrt{K_j}\bigl(\norm{\widehat y_{t,j}-z_j}_2+\norm{z_j-u_j}_2\bigr)
       +\norm{\eta_t-\nu_t}_2\\
    &\le M_2\varepsilon+\bigl(\sqrt{K_{\max}}\,H_1+1\bigr)\rho
    \le\frac{p_0}{4(K_{\max}+1)}+\frac{p_0}{32(K_{\max}+1)},
\end{align*}
where the last inequality uses
$\sqrt{K_{\max}}\,H_1+1\le\sqrt{K_{\max}+1}\,(H_1+1)$. If $a\in\mathcal A_j$,
then $u_{ja}\ge\widehat y_{t,ja}-\delta$. If $a=0$, then
$\nu_{t,j}-\one^Tu_j\ge\eta_{t,j}-\one^T\widehat y_{t,j}-\delta$. In both
cases, \eqref{eq:selected-action-lower-bound} shows that the left side is at
least
\begin{equation*}
    \frac{p_0}{K_{\max}+1}-\frac{9p_0}{32(K_{\max}+1)}
    =\frac{23p_0}{32(K_{\max}+1)}
    >\frac1{s_t},
\end{equation*}
where the last inequality holds because $s_t>\bar s\ge16(K_{\max}+1)/p_0$.

If $a\in\mathcal A_j$, then, since $u\in\cP(c_t,\nu_t)$ and both $W$ and $u$
are nonnegative, $w_{ja}u_{ja}\le Wu\le c_t$ coordinatewise, and hence
$C_t=s_tc_t\ge s_tu_{ja}w_{ja}\ge w_{ja}$. If $a=0$, the selected action
consumes no resources. In both cases, the selected action fits within the
remaining capacity and the policy sets $U_t=a$. By the scaling in
Section~\ref{subsec:one-period-loss}, $s_tu$ is an optimal solution of the LP
defining $\Phi(C_t,N_t)$, and $(s_tu)_{ja}>1$: if $a\in\mathcal A_j$, this
follows from $u_{ja}>1/s_t$, and if $a=0$, then
$(s_tu)_{j0}=N_{t,j}-\one^T(s_tu)_j=s_t(\nu_{t,j}-\one^Tu_j)>1$, because
$N_t=s_t\nu_t$. Lemma~\ref{lem:one-period-loss-properties}(i), applied with
$C=C_t$ and $N=N_t$, then gives $R_t^\mu=0$.

For the second claim, let $k\ge1$. If $\mathcal G_t$ occurs and
$\widehat e_t\le\varepsilon$, then $R_t^\mu=0$. If $\mathcal G_t$ occurs and
$\widehat e_t>\varepsilon$, then $R_t^\mu\le\bar R\le\bar R(\widehat e_t/\varepsilon)^{2k}$.
On $\mathcal G_t^c$, we have $R_t^\mu\le\bar R$. These three cases prove
\eqref{eq:loss-moment-bound}.
\end{proof}

The proposition compares $\widehat y_t$ with the optimal solution set of the
fluid LP itself, so it places no condition on the penalty parameter. The penalty
affects the loss only through the distance $\widehat e_t$, which we bound next.

\subsection{Tracking the optimal set of the fluid LP}\label{subsec:gap-potential}

We now prove that a high moment of $\widehat e_t$ has a bounded sum over the
horizon. The planned overuse of the iterates,
\begin{equation*}
    \zeta_t:=\norm{\pos{Wy_t^0-c_t}}_1,
    \qquad
    \widehat\zeta_t:=\norm{\pos{W\widehat y_t-c_t}}_1,
\end{equation*}
enters this bound through Lemma~\ref{lem:reward-feasibility}. We first bound how
the optimal solution set and the planned overuse change between periods. We then
show that the planned overuse is of order $\kappa_t$, bound the expected change
in the distance produced by one update, and sum the resulting recursion.

For each resource $i$, let $\bar w_i:=\max_{j,a}(w_{ja})_i$ be the maximum
consumption of resource $i$ by any action, and, for $c\in\R_+^d$, let $\bar c$
be the vector with components $\bar c_i:=\min\{c_i,\bar w_i\}$. If $\eta$ is a
probability vector, then every $y\in\mathcal Y(\eta)$ satisfies
$Wy\le\bar w$, since $\one^Ty_j\le\eta_j$ for every $j$.
A capacity component $c_i\ge\bar w_i$ is therefore never binding, and
\begin{equation}\label{eq:capacity-clipping}
    \pos{Wy-c}=\pos{Wy-\bar c},\quad y\in\mathcal Y(\eta),
    \qquad
    \mathcal O(c,\eta)=\mathcal O(\bar c,\eta).
\end{equation}
In particular, $0\le\zeta_t,\widehat\zeta_t\le\norm{\bar w}_1$. We write
$\bar c_t$ for the clipped per-period capacity in period $t$.

\begin{lemma}\label{lem:target-drift}
For every $t<T$, on every sample path,
\begin{equation}\label{eq:tracking-recursion}
    e_{t+1}\le\widehat e_t+\frac{D_e}{\ell_t},
    \qquad
    \zeta_{t+1}\le\widehat\zeta_t+\frac{D_v}{\ell_t},
\end{equation}
where $D_e:=M_1\bigl((1+H_1)D+2H_1\norm{\bar w}_2\bigr)$ and
$D_v:=\sqrt d\,\bigl(D\norm{W}_2+2\norm{\bar w}_2\bigr)$.
\end{lemma}

\begin{proof}
We first show that
\begin{align}
    \norm{y_{t+1}^0-\widehat y_t}_2&\le\frac{D}{\ell_t},\label{eq:warm-start-change}\\
    \norm{\bar c_{t+1}-\bar c_t}_2&\le\frac{\norm{\bar w}_2}{s_t-1}.\label{eq:capacity-change}
\end{align}
For every type $j$, we have $\widehat y_{t,j}=\eta_{t,j}\widehat q_{t,j}$ and
$y_{t+1,j}^0=\eta_{t+1,j}\widehat q_{t,j}$, where $\widehat q_{t,j}$ has
nonnegative entries with sum at most one, so that
$\norm{\widehat q_{t,j}}_2\le1$. The first identity also holds when
$\eta_{t,j}=0$, because both sides are then zero. Therefore,
$\norm{y_{t+1}^0-\widehat y_t}_2^2\le\sum_{j=1}^m(\eta_{t+1,j}-\eta_{t,j})^2$,
and \eqref{eq:warm-start-change} follows from
Assumption~\ref{ass:forecast-change}.

Let $w:=w_{J_tU_t}$ be the consumption of the executed action. Since this action
fits within the remaining capacity, $C_t\ge w$, and
$c_{t+1}=(C_t-w)/(s_t-1)$. Fix a resource $i$. If $c_{t,i}\ge\bar w_i$, then
$C_{t,i}-w_i\ge s_t\bar w_i-\bar w_i=(s_t-1)\bar w_i$, so $c_{t+1,i}\ge\bar w_i$
and $\bar c_{t+1,i}=\bar c_{t,i}=\bar w_i$. Otherwise, $0\le c_{t,i}<\bar w_i$
and
\begin{equation*}
    c_{t+1,i}-c_{t,i}
    =\frac{C_{t,i}-w_i}{s_t-1}-\frac{C_{t,i}}{s_t}
    =\frac{c_{t,i}-w_i}{s_t-1}.
\end{equation*}
Since both $c_{t,i}$ and $w_i$ lie in $[0,\bar w_i]$, it follows that
$|c_{t+1,i}-c_{t,i}|\le\bar w_i/(s_t-1)$. Because $v\mapsto\min\{v,\bar w_i\}$
is $1$-Lipschitz, $|\bar c_{t+1,i}-\bar c_{t,i}|\le\bar w_i/(s_t-1)$ in both
cases. This proves \eqref{eq:capacity-change}.

To prove the first inequality in \eqref{eq:tracking-recursion}, choose
$z\in\mathcal O(c_t,\eta_t)$ nearest to $\widehat y_t$ in the norm
$\norm\cdot_{L/\vartheta}$, so that $\norm{\widehat y_t-z}_{L/\vartheta}=\widehat e_t$. By
\eqref{eq:capacity-clipping}, $\mathcal O(c_t,\eta_t)=\mathcal O(\bar c_t,\eta_t)$
and $\mathcal O(c_{t+1},\eta_{t+1})=\mathcal O(\bar c_{t+1},\eta_{t+1})$.
Lemma~\ref{lem:optimal-face-sensitivity}, applied to the fluid LPs with
capacities $\bar c_t$ and $\bar c_{t+1}$ and demands $\eta_t$ and $\eta_{t+1}$,
gives $z'\in\mathcal O(c_{t+1},\eta_{t+1})$ with
\begin{equation*}
    \norm{z-z'}_2
    \le H_1\bigl(\norm{\bar c_{t+1}-\bar c_t}_2+\norm{\eta_{t+1}-\eta_t}_2\bigr).
\end{equation*}
By \eqref{eq:norm-equivalence},
Assumption~\ref{ass:forecast-change}, and \eqref{eq:capacity-change},
\begin{equation*}
    \norm{z-z'}_{L/\vartheta}
    \le M_1H_1\left(\frac{D}{\ell_t}+\frac{\norm{\bar w}_2}{s_t-1}\right).
\end{equation*}
The triangle inequality, \eqref{eq:norm-equivalence}, and
\eqref{eq:warm-start-change} give
\begin{equation*}
    e_{t+1}
    \le\norm{y_{t+1}^0-\widehat y_t}_{L/\vartheta}
       +\norm{\widehat y_t-z}_{L/\vartheta}
       +\norm{z-z'}_{L/\vartheta}
    \le\widehat e_t
       +\frac{M_1(1+H_1)D}{\ell_t}
       +\frac{M_1H_1\norm{\bar w}_2}{s_t-1}.
\end{equation*}
For $t<T$, we have $s_t\ge2$ and $s_t\ge\ell_t$, so
$s_t-1\ge s_t/2\ge\ell_t/2$. This proves the first inequality in
\eqref{eq:tracking-recursion}.

For the second, \eqref{eq:capacity-clipping} at periods $t$ and $t+1$, the
$1$-Lipschitz property of the positive part in each coordinate, and the
inequality $\norm\cdot_1\le\sqrt d\norm\cdot_2$ on $\R^d$ give
\begin{align*}
    \zeta_{t+1}
    &=\norm{\pos{Wy_{t+1}^0-\bar c_{t+1}}}_1\\
    &\le\norm{\pos{W\widehat y_t-\bar c_t}}_1
       +\norm{W(y_{t+1}^0-\widehat y_t)}_1
       +\norm{\bar c_{t+1}-\bar c_t}_1\\
    &\le\widehat\zeta_t
       +\sqrt d\left(\frac{D\norm{W}_2}{\ell_t}+\frac{\norm{\bar w}_2}{s_t-1}\right).
\end{align*}
Since $s_t-1\ge\ell_t/2$, the last term is at most $D_v/\ell_t$, which proves
the second inequality in \eqref{eq:tracking-recursion}.
\end{proof}

The next three lemmas concern the random update. Recall that
$\mathcal F_{t-1}$ is the information available before period $t$. The penalty
parameter $\kappa_t$ is deterministic, the vectors $y_t^0$, $c_t$, and
$\eta_t$ are $\mathcal F_{t-1}$-measurable, and, by
\eqref{eq:block-distribution}, the type $I_t$ whose block is updated has
distribution $\vartheta$ given $\mathcal F_{t-1}$. The proofs of
Lemmas~\ref{lem:violation-moments} and~\ref{lem:period-gap-bound} use the
following consequence of the convexity of $z\mapsto z^n$ on $[0,\infty)$: for
$a,b\ge0$, $\delta>0$, and an integer $n\ge1$,
\begin{equation}\label{eq:power-young}
    (a+b)^n\le(1+\delta)^{n-1}a^n+(1+\delta^{-1})^{n-1}b^n.
\end{equation}

We first show that the planned overuse is of order $\kappa_t$.

\begin{lemma}\label{lem:violation-moments}
For every integer $k\ge1$, there is a finite constant $V_k$, independent of
$T$, such that
\begin{equation}\label{eq:violation-moments}
    \E\zeta_t^{2k}\le V_k\kappa_t^{2k},
    \qquad t=1,\ldots,T.
\end{equation}
\end{lemma}

Without the reward term, the update of the block of type $j$ is the positive
part of $y_j-W_j^T\pos{Wy-c}/L_j$, which is componentwise at most $y_j$, so it
can only decrease the consumption of every resource. When a resource is
overused, some coordinate accounts for a fixed fraction of its consumption, and
the type of this coordinate is updated with probability at least
$\vartheta_{\min}$;
hence the expected overuse contracts by a constant factor. The reward term and
the changes of $c_t$ and $\eta_t$ between periods add terms of order $\kappa_t$,
since $\ell_t^{-1}\le\kappa_t$. The proof, given in
Appendix~\ref{app:violation-moments}, solves the resulting recursion. It uses
the initial fractions only through $\zeta_1$, and any initial fractions give
$\zeta_1\le\norm{\bar w}_1$.

The next lemma bounds the expected change in a high moment of the distance to
the optimal solution set produced by one update at a fixed state.

\begin{lemma}\label{lem:moment-contraction}
Fix an integer $k\ge2$. There are constants $\alpha_k\in(0,1/2]$ and
$A_k<\infty$ such that the following holds for every $c\in\R_+^d$, every
probability vector $\eta$, every $y\in\mathcal Y(\eta)$, and every
$\kappa\in(0,1]$. If $I$ has distribution $\vartheta$, $e=\dist_{L/\vartheta}(y,\mathcal O(c,\eta))$,
and $\zeta=\norm{\pos{Wy-c}}_1$, then
\begin{equation}\label{eq:moment-contraction}
    \E_I\dist_{L/\vartheta}^{2k}\bigl(\PG_{c,\eta}^{\kappa,I}(y),\mathcal O(c,\eta)\bigr)
    \le(1-\alpha_k\kappa)e^{2k}
       +A_k\left(\kappa^{k+1}+\frac{\zeta^{2k}}{\kappa^{k-1}}\right).
\end{equation}
\end{lemma}

The weights of the norm $\norm\cdot_{L/\vartheta}$ are chosen so that, in expectation
over the updated type, the change in the squared distance to a fixed point
$z\in\mathcal O(c,\eta)$ is the sum of the changes produced by updating each
block. The reward term then contributes $2\kappa(r^Ty-V_r(c,\eta))$, which
Lemma~\ref{lem:reward-feasibility} bounds by $-2\sigma\kappa e$ up to the
overuse, while the overuse term contributes a nonpositive amount because $z$
satisfies the capacity constraint. This decrease is linear in $e$, because the
error bound in Lemma~\ref{lem:reward-feasibility} is linear. Passing to the
$2k$th moment and using $e\le 2M_1$ turns it into the factor $1-\alpha_k\kappa$,
at the cost of the additive terms in \eqref{eq:moment-contraction}. The proof is
given in Appendix~\ref{app:moment-contraction}.

We now combine Lemmas~\ref{lem:target-drift}--\ref{lem:moment-contraction}.

\begin{lemma}\label{lem:period-gap-bound}
For the exponent $\gamma\in(0,1)$ in \eqref{eq:penalty-schedule}, choose an
integer
\begin{equation}\label{eq:moment-order}
    k>\max\left\{\frac1\gamma,\frac1{2(1-\gamma)}\right\}.
\end{equation}
Then there is a finite constant $\bar D_k$, independent of $T$, such that, for
every $T\ge1$,
\begin{equation}\label{eq:distance-sum}
    \sum_{t=1}^{T}\E e_t^{2k}\le\bar D_k
    \qquad\text{and}\qquad
    \sum_{t=1}^{T}\E\widehat e_t^{\,2k}\le\bar D_k.
\end{equation}
\end{lemma}

\begin{proof}
Since the maximum in \eqref{eq:moment-order} is at least $3/2$, we have
$k\ge2$, as Lemma~\ref{lem:moment-contraction} requires. By
\eqref{eq:moment-order}, $k\gamma>1$ and $2k(1-\gamma)>1$. Let
$u_t:=\E e_t^{2k}$ and $\widehat u_t:=\E\widehat e_t^{\,2k}$ for $t\in[T]$;
both are at most $(2M_1)^{2k}$.

\medskip\noindent\emph{Step 1: A recursion for the $2k$th moment.}
Since $\kappa_t\le1$ is deterministic, $y_t^0$, $c_t$, and $\eta_t$ are
$\mathcal F_{t-1}$-measurable, and $I_t$ has distribution $\vartheta$ given
$\mathcal F_{t-1}$, Lemma~\ref{lem:moment-contraction} with $\kappa=\kappa_t$
and $I=I_t$ gives
\begin{equation*}
    \E\bigl[\widehat e_t^{\,2k}\mid\mathcal F_{t-1}\bigr]
    \le(1-\alpha_k\kappa_t)e_t^{2k}
       +A_k\left(\kappa_t^{k+1}+\frac{\zeta_t^{2k}}{\kappa_t^{k-1}}\right).
\end{equation*}
We do not condition this inequality on $\mathcal G_t$, because $\mathcal G_t$
depends on the current and future requests through $\nu_t$. Taking expectations
and applying Lemma~\ref{lem:violation-moments}, we obtain
\begin{equation}\label{eq:post-update-moment}
    \widehat u_t\le(1-\alpha_k\kappa_t)u_t+A_k(1+V_k)\kappa_t^{k+1},
    \qquad t\in[T].
\end{equation}
Next, let $t<T$. By Lemma~\ref{lem:target-drift},
$0\le e_{t+1}\le\widehat e_t+D_e/\ell_t$. We apply \eqref{eq:power-young} with
$n=2k$ and $\delta_t:=\alpha_k\kappa_t/[4(2k-1)]\le1$. Since
$\alpha_k\kappa_t\le1/2$ and $\exp(v/4)\le1+v/2$ for $0\le v\le1$,
\begin{equation*}
    (1+\delta_t)^{2k-1}\le\exp(\alpha_k\kappa_t/4)\le1+\alpha_k\kappa_t/2,
    \qquad
    (1+\delta_t)^{2k-1}(1-\alpha_k\kappa_t)\le1-\frac{\alpha_k}2\kappa_t,
\end{equation*}
and $(1+\delta_t^{-1})^{2k-1}\le(2/\delta_t)^{2k-1}$. Taking expectations and
using \eqref{eq:post-update-moment}, we obtain
\begin{equation}\label{eq:time-varying-moment-recursion}
    u_{t+1}
    \le\left(1-\frac{\alpha_k}2\kappa_t\right)u_t
       +\Gamma_k\left(\kappa_t^{k+1}+\frac{\ell_t^{-2k}}{\kappa_t^{2k-1}}\right),
    \qquad t=1,\ldots,T-1,
\end{equation}
where
$\Gamma_k:=\max\{2A_k(1+V_k),[8(2k-1)/\alpha_k]^{2k-1}D_e^{2k}\}$.

\medskip\noindent\emph{Step 2: Summing the recursion.}
The coefficient of $u_t$ in \eqref{eq:time-varying-moment-recursion} is not
bounded away from one uniformly in $T$, so we do not sum a geometric series.
Instead, we rearrange \eqref{eq:time-varying-moment-recursion} and divide it by
$\kappa_t=\ell_t^{-\gamma}$:
\begin{equation}\label{eq:weighted-moment-recursion}
    \frac{\alpha_k}2u_t
    \le\ell_t^{\gamma}(u_t-u_{t+1})
       +\Gamma_k\left(\ell_t^{-k\gamma}+\ell_t^{-2k(1-\gamma)}\right),
    \qquad t=1,\ldots,T-1.
\end{equation}
A positive integer $n$ equals $\ell_t$ only if $t=n$ or $t=T-n+1$, so it occurs
at most twice among $\ell_1,\ldots,\ell_{T-1}$. Since $k\gamma>1$ and $2k(1-\gamma)>1$,
\begin{equation*}
    \sum_{t=1}^{T-1}\left(\ell_t^{-k\gamma}+\ell_t^{-2k(1-\gamma)}\right)
    \le S_k:=2\sum_{n=1}^{\infty}\left(n^{-k\gamma}+n^{-2k(1-\gamma)}\right)
    <\infty.
\end{equation*}
Suppose that $T\ge2$. We sum \eqref{eq:weighted-moment-recursion} over $t$ and
use $\ell_1=1$. Summation by parts gives
\begin{equation}\label{eq:weighted-telescoping}
\begin{aligned}
    \frac{\alpha_k}2\sum_{t=1}^{T-1}u_t
    \le{}&u_1
      +\sum_{t=2}^{T-1}\bigl(\ell_t^{\gamma}-\ell_{t-1}^{\gamma}\bigr)u_t\\
    &-\ell_{T-1}^{\gamma}u_T
      +\Gamma_kS_k.
\end{aligned}
\end{equation}
The term $-\ell_{T-1}^{\gamma}u_T$ is nonpositive. For $t\ge2$, if $t\le s_t$,
then $\ell_t=t$ and $\ell_{t-1}=t-1$; otherwise, $\ell_t=s_t\le\ell_{t-1}$.
Thus, each increment $\ell_t^{\gamma}-\ell_{t-1}^{\gamma}$ is either
nonpositive or equal to $t^{\gamma}-(t-1)^{\gamma}$. By the mean value theorem,
$n^{\gamma}-(n-1)^{\gamma}\le\gamma(n-1)^{\gamma-1}$ for $n\ge2$, which tends to
zero because $\gamma<1$. Let $n_0\ge2$ be an integer such that
$n^{\gamma}-(n-1)^{\gamma}\le\alpha_k/4$ for every $n\ge n_0$; it depends on the
problem instance, $\gamma$, $k$, the constants $L_j$, and $\vartheta$, but not
on $T$. We
discard the nonpositive increments, bound $u_t$ by $(2M_1)^{2k}$ when $t<n_0$,
and bound the increment by $\alpha_k/4$ when $t\ge n_0$. This argument does not
require $\ell_t$ to reach $n_0$, so it also applies to short horizons. We obtain
\begin{equation*}
\begin{aligned}
    \sum_{t=2}^{T-1}\bigl(\ell_t^{\gamma}-\ell_{t-1}^{\gamma}\bigr)u_t
    &\le (2M_1)^{2k}\sum_{t=2}^{n_0-1}\bigl(t^{\gamma}-(t-1)^{\gamma}\bigr)
      +\frac{\alpha_k}4\sum_{t=1}^{T-1}u_t\\
    &\le (2M_1)^{2k}\bigl(n_0^{\gamma}-1\bigr)
      +\frac{\alpha_k}4\sum_{t=1}^{T-1}u_t,
\end{aligned}
\end{equation*}
where the sum over $t<n_0$ telescopes to $(n_0-1)^\gamma-1\le n_0^\gamma-1$.
Substituting this bound and
$u_1\le (2M_1)^{2k}$ into \eqref{eq:weighted-telescoping} gives
\begin{equation*}
    \frac{\alpha_k}4\sum_{t=1}^{T-1}u_t
    \le (2M_1)^{2k}n_0^{\gamma}+\Gamma_kS_k.
\end{equation*}
Since $u_T\le (2M_1)^{2k}$,
\begin{equation*}
    \sum_{t=1}^{T}u_t
    \le (2M_1)^{2k}+\frac4{\alpha_k}\bigl((2M_1)^{2k}n_0^{\gamma}+\Gamma_kS_k\bigr).
\end{equation*}
This bound also holds when $T=1$. Finally, summing
\eqref{eq:post-update-moment} over $t$ and using
$\sum_{t=1}^T\kappa_t^{k+1}\le2\sum_{n=1}^\infty n^{-(k+1)\gamma}<\infty$, which
holds because $(k+1)\gamma>k\gamma>1$, shows that $\sum_{t=1}^T\widehat u_t$ is
bounded by the preceding bound plus $2A_k(1+V_k)\sum_{n=1}^\infty n^{-(k+1)\gamma}$.
Taking $\bar D_k$ to be this sum proves \eqref{eq:distance-sum}.
\end{proof}

Lemma~\ref{lem:period-gap-bound} only requires $k$ to exceed the maximum in
\eqref{eq:moment-order}, so $k$ can be chosen arbitrarily large. However, the
constant $C$ in the proof of Theorem~\ref{thm:regret} in
Section~\ref{subsec:regret-proof} depends on $k$ through the factor
$\varepsilon^{-2k}$ and the constant $\bar D_k$, and a large $k$ can make it
much larger. The smallest choice, $k=2$, is possible when $1/2<\gamma<3/4$. In
the numerical experiments in Section~\ref{sec:numerical-experiments}, we use
$\gamma=0.6$.

\subsection{Proof of Theorem \ref{thm:regret}}\label{subsec:regret-proof}
We now combine the preceding loss and distance bounds with
Assumption~\ref{ass:forecast-accuracy}.

\begin{proof}[Proof of Theorem~\ref{thm:regret}]
Fix an integer $k$ that satisfies \eqref{eq:moment-order}. Write $R_t=R_t^\mu$ for the policy $\mu$ in
Algorithm~\ref{alg:first-order-policy}. By
Lemma~\ref{lem:one-period-loss-properties}(ii), $R_t\le\bar R$ in every period.
If $T\le\bar s$, Lemma~\ref{lem:loss-telescoping} therefore gives
$\Reg_T(\mu)\le T\bar R\le\bar s\bar R$.

Suppose that $T>\bar s$. For $t\le T-\bar s$, we have $s_t=T-t+1>\bar s$, so
Proposition~\ref{prop:loss-gap} gives
\begin{equation*}
    R_t
    \le\bar R\,\varepsilon^{-2k}\,\widehat e_t^{\,2k}
       +\bar R\,\one_{\mathcal G_t^c},
    \qquad t=1,\ldots,T-\bar s.
\end{equation*}
Each of the final $\bar s$ periods contributes at most $\bar R$. Taking
expectations in Lemma~\ref{lem:loss-telescoping} and applying
Lemma~\ref{lem:period-gap-bound} and Assumption~\ref{ass:forecast-accuracy}, we
obtain
\begin{align*}
\Reg_T(\mu)
&\le
\bar R\,\varepsilon^{-2k}\sum_{t=1}^{T-\bar s}\E\widehat e_t^{\,2k}
+\bar R\sum_{t=1}^{T-\bar s}\Prob(\mathcal G_t^c)
+\bar s\bar R\\
&\le
\bar R\left(
    \varepsilon^{-2k}\bar D_k
    +\Lambda(\rho)
    +\bar s
\right)
=:C.
\end{align*}
The constant $C$ is independent of $T$ and is at least $\bar s\bar R$.
Therefore, $\Reg_T(\mu)\le C$ for every $T\ge1$.
\end{proof}

\begin{remark}\label{rem:implementation-variants}
The proof of Theorem~\ref{thm:regret} uses the initial fractions only through
the planned overuse $\zeta_1$ in the first period, in
Lemma~\ref{lem:violation-moments}, and it uses the executed action only through
its feasibility, in Lemma~\ref{lem:target-drift}, and through the fact that a
selected action that fits is executed, in Proposition~\ref{prop:loss-gap}.
Hence the theorem continues to hold if $q_1$ is any
vector of fractions with nonnegative entries and block sums at most one, and if
an infeasible selected action is replaced by any action that fits within the
remaining capacity, for example the feasible action with the largest updated
fraction. Finally, dividing all rewards by
a positive constant before running the policy, as in
Section~\ref{sec:numerical-experiments}, gives
Algorithm~\ref{alg:first-order-policy} for the rescaled rewards, to which the
theorem applies; the regret in the original reward units is the same constant
times the regret in the rescaled units.
\end{remark}

\section{Numerical experiments}\label{sec:numerical-experiments}

We evaluate our policy on the dynamic fulfillment problem studied by \citet{jasin_sinha_2015} and \citet{ma_fulfillment_2023}. We compare it with first-order and LP-based policies that have the same information about the arrival distribution, first when the arrival probabilities are known and then when they are unknown, at horizons from $T=1{,}000$ to $T=100{,}000$.

\subsection{Fulfillment instances}\label{sec:setup-numerical-experiments}

An online retailer fulfills orders from a network of centers over $T$ periods. The network consists of the ten customer regions and five fulfillment centers of the small-network experiments of \citet{ma_fulfillment_2023}. Each center starts with a fixed inventory of each item, and there is no replenishment. In each period at most one order arrives; it specifies a customer region and a basket of distinct items, one unit of each. After observing the order and the remaining inventories, the retailer immediately decides which items to fulfill and from which centers. An order may be split across centers, partially fulfilled, or rejected, and each shipped unit reduces the inventory of that item at its center by one.

We use the same cost structure as \citet{ma_fulfillment_2023}. Each center used for an order incurs a package cost of $8.759$, and each shipped item incurs a shipping cost of $0.423$ plus $0.000541$ per mile of distance from the center to the region. If any item of an order is not fulfilled, the order incurs a cost of $2\times 8.759$ and, for each unfulfilled item, twice the largest shipping cost from a center to the region. We also follow \citet{ma_fulfillment_2023} in setting the initial inventories. Each center carries each item independently with probability $0.75$. The initial inventory of an item at a center equals $T\omega+0.5\sqrt{T\omega(1-\omega)}$, rounded to the nearest integer, where $\omega$ is the probability that the request in a period is an order that contains the item and comes from a region whose closest center carrying the item is this center. This nearest-center assignment only determines the inventories; a policy may ship an item from any center that has it in stock. Because of the square-root safety stock, the per-period inventory $B_T/T$ depends on $T$ and approaches the nearest-center demand rates as $T$ grows.

We consider 30 instances indexed by $(n,M)$, where $n\in\{10,20,30,40,50,60\}$ is the number of items and $M\in\{1,\ldots,5\}$ is the largest number of items in an order. For each $n$, the basket catalogue contains the $n$ single-item baskets and, for each size from two to five, five distinct subsets drawn uniformly at random; each basket receives an independent weight drawn uniformly from $[0,1]$. The instance with maximum order size $M$ keeps the baskets with at most $M$ items. In each period, a region is drawn with probability proportional to its population. Independently, no order arrives with probability $0.1$; otherwise the order has $s$ items with probability proportional to $2^{-(s-1)}$, $s=1,\ldots,M$, and the basket is drawn among those of size $s$ with probability proportional to its weight. For a fixed $n$, the five instances with $M=1,\ldots,5$ share the basket catalogue, the weights, and the set of centers carrying each item.

In the notation of Section~\ref{sec:model}, a request type is a basket--region pair and a resource is an item--center pair. An action of a type assigns each item of the basket either to a center or to being unfulfilled, and it consumes one unit of each item--center pair that it uses. We keep only the actions whose cost is below the cost of rejecting the order and that use item--center pairs with positive inventory; the reward of an action is the cost of rejecting the order minus the cost of the action. The instances have between 110 and 810 request types, including one type per region for periods without an order, whose only action is rejection; between 31 and 225 item--center pairs with positive inventory; and between 310 and 38,581 actions other than rejection, with up to 1,164 for a single request type. We divide all rewards by the largest reward of the instance, which is determined by the catalogue before any demand is observed. For each instance we draw 30 independent demand sequences of length 100,000 and use their first $T$ requests for $T\in\{1{,}000, 5{,}000, 10{,}000, 20{,}000, 50{,}000, 100{,}000\}$, with the initial inventory for each $T$ given by the rule above. Tables~\ref{tab:numerical-known} and~\ref{tab:numerical-unknown} report four of these horizons, and Figure~\ref{fig:numerical-regret} reports all six. All policies face the same demand sequences.

\subsection{Policies and performance measures}\label{sec:numerical-policies}

The arrivals in this section are i.i.d. Our policy is Algorithm~\ref{alg:first-order-policy} with $I_t=J_t$ and the exponent $\gamma=0.6$. The constants $L_j$ are given by the larget eigenvalue of $W_j\bigl(I-\one\one^T/(K_j+1)\bigr)W_j^T$ if $W_j\ne0$ and 1 otherwise, where the identity matrix and the vector of ones have dimension $K_j$.
We use the following two implementation choices, both covered by Remark~\ref{rem:implementation-variants}.
\begin{enumerate}[label=(\roman*),leftmargin=*,itemsep=2pt,topsep=3pt]
\item The starting solution assigns the fraction $\min\{1,\min_i (B_{T,i}/T)/\Lambda_i\}$ of the demand of each type to its highest-reward action, where the inner minimum is over the resources $i$ that this action uses, and the remaining fraction to rejection. Here $\Lambda_i$ is the expected per-period consumption of resource $i$ when every type uses its highest-reward action, computed with $p$ when $p$ is known and with equal probabilities otherwise. When $p$ is known, we also run our policy from an LP start. It solves the fluid LP with demand $Tp$ and inventory $B_T$ once, before the first period, and sets $q_{1,ja}=x^*_{ja}/(Tp_j)$ for an optimal solution $x^*$, with the remaining fraction of each type assigned to rejection.
\item We use a fallback rule. If the remaining inventory cannot serve the selected action, the policy forms the price vector $\lambda_t=\pos{W\widehat y_t-c_t}/\kappa_t$ at the updated allocation $\widehat y_t$, so that $r_{ja}-w_{ja}^T\lambda_t$ is the partial derivative of the penalized objective $f_{c_t}^{\kappa_t}$ in \eqref{eq:penalized-objective} with respect to $y_{ja}$ at $\widehat y_t$. Among the actions of the arriving type that the remaining inventory can serve, it executes the one with the largest positive value of $r_{ja}-w_{ja}^T\lambda_t$. If no such action exists, it rejects the request. The rule does not change the update.
\end{enumerate}
When the arrival probabilities are known, the forecast is $\eta_t=p$, as in Model 2; when they are unknown, it is $\pi_t$, as in Model 1. A type whose only action is rejection has an empty block, which is never updated. By Corollary~\ref{cor:arrival-models} and Remark~\ref{rem:implementation-variants}, the regret bound applies to this implementation with either starting solution and in both information settings.

We compare our policy with the following policies, using the parameters stated below without tuning. The price-based policies select, for an arriving request of type $j$, the action with the largest positive value of $r_{ja}-w_{ja}^T\lambda$ at the current price vector $\lambda$, rejecting the request if no value is positive, and they update their prices with the selected action. In their original form, these policies reject the request when the remaining inventory cannot serve the selected action. We run them with the same kind of fallback rule as our policy: among the actions of the arriving type that the remaining inventory can serve, the policy executes the one with the largest positive value of $r_{ja}-w_{ja}^T\lambda$ at the current prices, and it rejects the request if no such action exists. The price update still uses the selected action. We also ran the original versions without the fallback rule. Their mean percentage loss was higher at every horizon. We do not include the policies of \citet{li_wang_zhang_2024} and \citet{gao_et_al_2026}, which are designed for problems in which each request is accepted or rejected.

With known arrival probabilities, the comparison includes:
\begin{itemize}[leftmargin=*,itemsep=2pt,topsep=3pt]
\item \citet{vera_banerjee_2021}, Algorithm~4. In each period $t$, the policy solves the fluid LP with the remaining inventory $C_t$ and demand $s_tp$. It then takes the action of the arriving type with the largest LP allocation, including rejection, and breaks ties in favor of rejection. Because the LP uses the remaining inventory, the selected action can always be served.
\item \citet{he_wei_xu_yu_2025}, Algorithm~1, with the fallback rule. The prices start at the optimal dual solution of the fluid LP with demand $Tp$ and inventory $B_T$ and are updated by $\lambda\leftarrow[\lambda+(w_{ja}-B_T/T)/V]_+$ with $V=\max\{T,\max_{j,a}\|w_{ja}\|_2^2+\|B_T/T\|_2^2\}$.
\end{itemize}
With unknown arrival probabilities, the comparison includes:
\begin{itemize}[leftmargin=*,itemsep=2pt,topsep=3pt]
\item DMD, Algorithm~1 of \citet{balseiro_lu_mirrokni_2023} with the Euclidean mirror map, with the fallback rule. The prices start at zero and are updated by $\lambda\leftarrow[\lambda+T^{-1/2}(w_{ja}-B_T/T)]_+$. On 30 other instances generated by the same procedure with different random draws, the step-size grid of \citet[Appendix~H.1]{balseiro_lu_mirrokni_2023}, $s/\sqrt{Td}$ with $s\in\{0.1,1,10,100\}$ and $d$ the number of item--center pairs with positive inventory, did not improve on the step $T^{-1/2}$: its best value, $s=10$, earned less total reward at every horizon.
\item \citet{ma_cao_tsang_xia_2025}, Algorithm~5 without the regularizer, with the fallback rule. It makes one price update per period. The prices start at zero and are updated by $\lambda\leftarrow[\lambda+\gamma_t(w_{ja}-\bar c)]_+$ with $\gamma_t=0.1/(t-\ell+1)$ and the epoch factor $1/2$ of the authors' code; at each epoch boundary $T-\lceil T/2^k\rceil$, the target $\bar c$ is reset to the remaining inventory divided by the number of remaining periods and $\ell$ is set to the current period.
\end{itemize}

In this section, regret refers to $\Phi(B_T,N)-\REV^\mu_T$, the regret against the hindsight LP of Section~\ref{sec:model}; its expectation is an upper bound on the regret defined in Section~\ref{sec:model}. The regret is measured in the normalized reward units. We also report the percentage loss $100\,(\Phi(B_T,N)-\REV^\mu_T)/\Phi(B_T,N)$. For each policy, instance, and horizon, we average over the 30 demand sequences and report the mean, minimum, and maximum of the 30 instance averages. We compute $\Phi(B_T,N)$ by column generation with HiGHS, certifying optimality over all actions in the menus. The code for all experiments is available in the replication package.

\subsection{Results}\label{sec:numerical-results}

\begin{table}[!htb]\centering
\caption{Known arrival probabilities. Percentage loss relative to the hindsight LP.}
\label{tab:numerical-known}
\begingroup\small\setlength{\tabcolsep}{2.4pt}
\begin{tabular*}{\linewidth}{@{}l@{\hspace{6pt}\extracolsep{\fill}}*{12}{r}@{}}\toprule
 & \multicolumn{3}{c}{$T=1{,}000$} & \multicolumn{3}{c}{$T=5{,}000$} & \multicolumn{3}{c}{$T=10{,}000$} & \multicolumn{3}{c}{$T=100{,}000$} \\
\cmidrule(lr){2-4}\cmidrule(lr){5-7}\cmidrule(lr){8-10}\cmidrule(lr){11-13}
Policy & Mean & Min & Max & Mean & Min & Max & Mean & Min & Max & Mean & Min & Max \\\midrule
Our policy & 1.404 & 0.108 & 2.616 & 0.795 & 0.035 & 1.699 & 0.600 & 0.022 & 1.444 & 0.191 & 0.005 & 0.776 \\
Our policy, LP start & 1.020 & 0.108 & 2.006 & 0.468 & 0.035 & 0.910 & 0.318 & 0.022 & 0.610 & 0.070 & 0.005 & 0.175 \\
\citet{vera_banerjee_2021} & \textbf{0.758} & 0.179 & 1.278 & \textbf{0.293} & 0.100 & 0.506 & \textbf{0.187} & 0.058 & 0.324 & \textbf{0.035} & 0.012 & 0.064 \\
\citet{he_wei_xu_yu_2025} & 1.149 & 0.110 & 2.149 & 0.596 & 0.044 & 1.219 & 0.437 & 0.030 & 0.814 & 0.140 & 0.010 & 0.283 \\
\bottomrule\end{tabular*}
\par\vspace{3pt}\raggedright\footnotesize
Mean, minimum, and maximum over the 30 instances of the instance-average percentage loss; each instance average is taken over the 30 demand sequences. The lowest mean at each horizon is in bold.
\par\endgroup\end{table}

\begin{table}[!htb]\centering
\caption{Unknown arrival probabilities. Percentage loss relative to the hindsight LP.}
\label{tab:numerical-unknown}
\begingroup\small\setlength{\tabcolsep}{3pt}
\begin{tabular*}{\linewidth}{@{}l@{\hspace{6pt}\extracolsep{\fill}}*{12}{r}@{}}\toprule
 & \multicolumn{3}{c}{$T=1{,}000$} & \multicolumn{3}{c}{$T=5{,}000$} & \multicolumn{3}{c}{$T=10{,}000$} & \multicolumn{3}{c}{$T=100{,}000$} \\
\cmidrule(lr){2-4}\cmidrule(lr){5-7}\cmidrule(lr){8-10}\cmidrule(lr){11-13}
Policy & Mean & Min & Max & Mean & Min & Max & Mean & Min & Max & Mean & Min & Max \\\midrule
Our policy & \textbf{1.639} & 0.148 & 2.849 & \textbf{0.963} & 0.051 & 1.923 & \textbf{0.720} & 0.030 & 1.578 & \textbf{0.206} & 0.005 & 0.607 \\
DMD & 2.099 & 0.455 & 3.651 & 1.336 & 0.273 & 2.258 & 1.036 & 0.207 & 1.736 & 0.392 & 0.075 & 0.696 \\
\citet{ma_cao_tsang_xia_2025} & 2.601 & 0.230 & 4.342 & 2.497 & 0.172 & 4.338 & 2.454 & 0.146 & 4.219 & 2.374 & 0.102 & 4.154 \\
\bottomrule\end{tabular*}
\par\vspace{3pt}\raggedright\footnotesize
Mean, minimum, and maximum over the 30 instances of the instance-average percentage loss; each instance average is taken over the 30 demand sequences. The lowest mean at each horizon is in bold.
\par\endgroup\end{table}

\begin{figure}[!htb]\centering
\includegraphics[width=\linewidth]{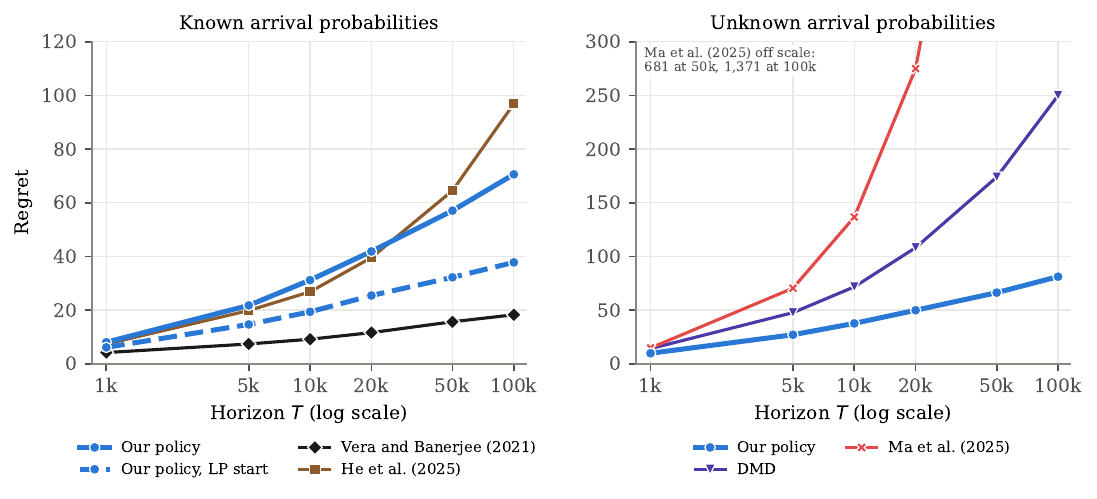}
\caption{Regret against the hindsight LP for the instance $(n,M)=(40,3)$, in normalized reward units, averaged over the 30 demand sequences. The horizontal axis uses a logarithmic scale; the two panels use different vertical ranges.}
\label{fig:numerical-regret}
\end{figure}

With known arrival probabilities, the policy of \citet{vera_banerjee_2021} has the lowest mean percentage loss at every horizon, our policy with the LP start the second lowest, and the policy of \citet{he_wei_xu_yu_2025} the third lowest (Table~\ref{tab:numerical-known}). The policy of \citet{vera_banerjee_2021} has lower loss than our policy with the LP start in 24 of the 30 instances at each horizon up to $T=50{,}000$ and in 23 at $T=100{,}000$. Our policy with the LP start has lower loss than \citet{he_wei_xu_yu_2025} in 22, 29, 29, 30, 30, and 30 of the 30 instances at $T=1{,}000$, $5{,}000$, $10{,}000$, $20{,}000$, $50{,}000$, and $100{,}000$. Without the LP start, our policy has lower loss than \citet{he_wei_xu_yu_2025} in only 4 to 17 of the 30 instances.

With unknown arrival probabilities, our policy has the lowest mean percentage loss at every horizon and DMD the second lowest (Table~\ref{tab:numerical-unknown}). Our policy has lower loss than DMD in 29, 29, 27, 27, 27, and 28 of the 30 instances at the six horizons. It has lower loss than \citet{ma_cao_tsang_xia_2025} in all 30 instances at every horizon.

Figure~\ref{fig:numerical-regret} shows the regret for the instance $(n,M)=(40,3)$; for every policy, the regret of this instance at $T=100{,}000$ is near the middle of the 30 instances. The regret of our policy increases over the tested horizons without a visible plateau. This is consistent with Theorem~\ref{thm:regret}, whose constant depends on the instance and can be large, but the experiments do not show at which horizon the regret stops increasing.

\subsection{A large instance}\label{sec:numerical-huge}

We next consider a much larger instance. We generate it on the large network of \citet{ma_fulfillment_2023}, with 99 customer regions and 10 fulfillment centers, by the procedure of Section~\ref{sec:setup-numerical-experiments} with $n=10{,}000$ items and 213{,}000 baskets of each size from two to five instead of five. It has 85.3 million request types, about 75{,}000 item--center pairs with positive inventory, and $9.98\times10^{10}$ type--action pairs. A policy only needs the menu of the arriving request, which it generates from the catalogue when the request arrives; menus have 193 actions on average under the arrival distribution and at most 17{,}540. The fluid LP has one demand constraint for each of the 85.3 million types, so we compare only the policies that solve no LP, our policy, DMD, and the policy of \citet{ma_cao_tsang_xia_2025}, all with unknown arrival probabilities and the fallback rule, at $T\in\{10^5,10^6,10^7,10^8\}$. To estimate regret or percentage loss, we compute an upper bound on the hindsight LP by finding a feasible solution of its dual. For every $\lambda\in\R_+^d$, weak duality gives
\begin{equation*}
    \Phi(B_T,N)\le \lambda^TB_T+\sum_{j=1}^m N_j\max\Bigl\{0,\max_{a\in\mathcal A_j}\bigl(r_{ja}-w_{ja}^T\lambda\bigr)\Bigr\}=:U,
\end{equation*}
and we measure loss against $U$, choosing $\lambda$ for each demand sequence by minimizing a smoothed version of the right-hand side. The resulting percentage loss $100\,(U-\REV^\mu_T)/U$ is an upper bound on the percentage loss relative to the hindsight LP; on smaller instances of the same family, where the hindsight LP can be solved, $U$ exceeded $\Phi(B_T,N)$ by less than $0.02\%$.

\begin{table}[!htb]\centering
\caption{Instance with $9.98\times10^{10}$ type--action pairs, unknown arrival probabilities. Percentage loss relative to an upper bound on the hindsight LP, and CPU time per request.}
\label{tab:numerical-huge}
\begingroup\small\setlength{\tabcolsep}{3pt}
\begin{tabular*}{\linewidth}{@{}l@{\extracolsep{\fill}}rrrrr@{}}\toprule
 & \multicolumn{4}{c}{Percentage loss} & CPU time \\
\cmidrule(lr){2-5}
Policy & $T=10^{5}$ & $T=10^{6}$ & $T=10^{7}$ & $T=10^{8}$ & ($\mu$s per request) \\\midrule
Our policy & \textbf{4.92} & \textbf{2.29} & \textbf{1.21} & \textbf{0.26} & 21--30 \\
DMD & 4.96 & 2.84 & 2.51 & 1.95 & 9--10 \\
\citet{ma_cao_tsang_xia_2025} & 5.06 & 3.13 & 3.47 & 3.97 & 8--10 \\
\bottomrule\end{tabular*}
\par\vspace{3pt}\raggedright\footnotesize
Mean over demand sequences (10, 10, 10, 5 sequences at the 4 horizons) of the percentage loss $100\,(U-\REV^\mu_T)/U$, where $U\ge\Phi(B_T,N)$ is the upper bound in the text; the largest range over sequences of any entry is 0.36 percentage points. CPU time is the range, over all runs, of the CPU time of a run divided by $T$. The lowest loss at each horizon is in bold.
\par\endgroup\end{table}

Our policy has the lowest loss at every horizon (Table~\ref{tab:numerical-huge}). At $T=10^5$, where almost every request is of a type that has not arrived before, the losses of the three policies are between $4.9\%$ and $5.1\%$. As $T$ increases to $10^8$, the loss of our policy falls to $0.26\%$, while that of DMD falls to $1.95\%$ and that of \citet{ma_cao_tsang_xia_2025} stays between $3.1\%$ and $4.0\%$ for $T\ge10^6$. The CPU time per request of every policy does not grow with $T$. It is between 21 and 30 microseconds for our policy and between 8 and 10 microseconds for the price-based policies, which only compare the actions of the arriving menu at the current prices.

\section{Concluding remarks}\label{sec:concluding-remarks}

We have presented a primal first-order learning policy for finite-action
online resource allocation that achieves $O(1)$ expected additive regret under
a general arrival model, which includes i.i.d.\ arrivals with unknown
probabilities, independent nonstationary arrivals with known probabilities or
with one historical sample path, and Markov-modulated arrivals with an unknown
Markov chain. The policy does not solve a linear program, and the regret bound
does not require nondegeneracy, because the analysis compares optimal solution
sets instead of tracking an optimal basis or an optimal dual solution.

Several variants of the online resource allocation model of this paper have
been studied in the literature, including allocation with pricing, in which
the price is posted after the type of the arriving customer is observed
\citep{vera_banerjee_gurvich_2021,vera_banerjee_gurvich_erratum_2026};
price-based revenue management, in which the prices are posted before the
customer arrives \citep{wang_wang_2022}; multiway dynamic matching
\citep{wei_xu_yu_2026}; delayed decisions \citep{xie_ma_xin_2025}; and an
uncertain horizon \citep{balseiro_kroer_kumar_2023,banerjee_freund_2025}. Our
model also assumes that the initial capacity is fixed and that there is no
replenishment during the horizon. These assumptions have been relaxed
recently \citep{arlotto_keskin_wei_2026,chen_et_al_2026,ling_jiang_xin_2026}.
For several of these models, regret bounds have been established for policies
that solve linear programs, and a natural question is whether first-order
policies can achieve similar results.

Our model assumes that rewards and resource-consumption vectors have finite
support. When both the rewards and the resource consumptions are continuously
distributed, LP-based re-solving policies attain $O(\log T\log\log T)$ regret
\citep{li_ye_2022} and the tight rate $\Theta(\log T)$ \citep{bray_2025}, and
the first-order variant of \citet{ma_cao_tsang_xia_2025} attains
$O(\log^2 T)$ regret; these results assume nondegeneracy and local regularity
conditions. Without nondegeneracy, \citet{jiang_ma_zhang_2025} establish an
$O(\log^2 T)$ bound for NRM with continuously distributed rewards and
finite-support resource consumption by re-solving the fluid LP in every
period, and \citet{zhang_multisecretary_2026} shows that this bound is tight,
already for the multisecretary problem. \citet{zhang_continuous_2026} also
allows the resource consumption to be continuously distributed. Preliminary results from our ongoing work suggest
that a first-order learning policy can attain the $O(\log^2 T)$ bound of
\citet{jiang_ma_zhang_2025} without solving a linear program. It is open
whether a comparable bound can be obtained in the model of
\citet{zhang_continuous_2026} using only first-order updates.


\bibliographystyle{apalike}
\bibliography{references}

\clearpage
\appendix

\section{Proofs for Section~\ref{sec:regret-analysis}}
\subsection{Proof of Lemma~\ref{lem:one-period-loss-properties}}
\label{app:one-period-loss-properties}

\begin{proof}
Since $C-w_{ja}\ge0$ and $N-\mathbf e_j\ge0$, the zero vector belongs to
$\cP(C,N)$ and to $\cP(C-w_{ja},N-\mathbf e_j)$. Both sets are
compact, because every element $z$ of $\cP(C,N)$ satisfies $0\le z_{ib}\le N_i$
for every $i\in[m]$ and $b\in\mathcal A_i$, and similarly for $\cP(C-w_{ja},N-\mathbf e_j)$. Hence, both LPs have nonempty
compact optimal solution sets. Let $S:=\{z\in\cP(C,N):z_{ja}\ge1\}$ be the set
over which the maximum in \eqref{eq:fixed-action-value} is taken, so that
\begin{equation}\label{eq:fixed-action-value-set}
    r_{ja}+\Phi(C-w_{ja},N-\mathbf e_j)=\max_{z\in S}r^Tz,
\end{equation}
and the maximum is attained.

To prove part (i), suppose first that an optimal solution $z$ of the LP
defining $\Phi(C,N)$ belongs to $S$. Then the maximum in
\eqref{eq:fixed-action-value-set} is at least $r^Tz=\Phi(C,N)$. It is also at
most $\Phi(C,N)$ because $S\subseteq\cP(C,N)$. Therefore,
\eqref{eq:fixed-action-value-set} gives
$\Phi(C,N)-r_{ja}-\Phi(C-w_{ja},N-\mathbf e_j)=0$. Conversely, if this
difference is zero, then a maximizer $z$ in \eqref{eq:fixed-action-value-set}
belongs to $S$ and has $r^Tz=\Phi(C,N)$. Thus, $z$ is an optimal solution of
the LP defining $\Phi(C,N)$ that belongs to $S$.

To prove part (ii), note first that the maximum in
\eqref{eq:fixed-action-value-set} is at most $\Phi(C,N)$, so the difference is
nonnegative. Let $z^0$ be an optimal solution of the LP defining $\Phi(C,N)$.
Lemma~\ref{lem:optimal-face-sensitivity}, applied with $c=C$,
$c'=C-w_{ja}$, $\eta=N$, $\eta'=N-\mathbf e_j$, and $\widetilde r=r$, gives an
optimal solution $z^1$ of the LP defining $\Phi(C-w_{ja},N-\mathbf e_j)$ such that
\[
    \norm{z^0-z^1}_2
    \le H_1\sqrt{\norm{w_{ja}}_2^2+1}
    \le H_1\sqrt{w_{\max}^2+1}.
\]
Since $r^Tz^0=\Phi(C,N)$, $r^Tz^1=\Phi(C-w_{ja},N-\mathbf e_j)$, and
$r_{ja}\ge0$, the Cauchy--Schwarz inequality gives
\begin{equation*}
\Phi(C,N)-r_{ja}-\Phi(C-w_{ja},N-\mathbf e_j)
\le r^T(z^0-z^1)
\le \norm r_2H_1\sqrt{w_{\max}^2+1}
\le H_2(1+\norm r_2).
\end{equation*}
This proves part (ii).
\end{proof}

\subsection{Proof of Lemma~\ref{lem:violation-moments}}
\label{app:violation-moments}

\begin{proof}
The proof has three steps. We first show that, at a fixed state, the overuse
term alone reduces the expected overuse by a constant factor. We then add the
reward term and the change of state between periods. Finally, we solve the
resulting recursion.

\medskip\noindent\emph{Step 1: The overuse term at a fixed state.}
Fix $c\in\R_+^d$, a probability vector $\eta$, and $y\in\mathcal Y(\eta)$, and
let $v:=\pos{Wy-c}$ and $\zeta:=\norm v_1$. For $j\in[m]$, let $Q_jy$ be the
vector obtained from $y$ by replacing its block of type $j$ with
\begin{equation*}
    [Q_jy]_j:=\Pi_{\mathcal Y_j(\eta_j)}\bigl(y_j-W_j^Tv/L_j\bigr).
\end{equation*}
This is the update \eqref{eq:observed-type-update} without the reward term.
Since $W_j^Tv\ge0$, the vector $\pos{y_j-W_j^Tv/L_j}$ is componentwise at most
$y_j$. Hence, its entries sum to at most $\eta_j$, and it belongs to
$\mathcal Y_j(\eta_j)$. It is the projection of $y_j-W_j^Tv/L_j$ onto the
nonnegative orthant, so it is also the projection onto the smaller set
$\mathcal Y_j(\eta_j)$:
\begin{equation*}
    [Q_jy]_j=\pos{y_j-W_j^Tv/L_j}\le y_j.
\end{equation*}
Consequently, $WQ_jy\le Wy$, and the overuse of every resource does not
increase.

Let $\omega>0$ be the smallest positive entry of $W$, which exists because
some resource-consumption vector has a positive entry, and let
\begin{equation*}
    \tau:=\min\left\{\frac1K,\frac{\omega^2}{L_{\max}}\right\},
    \qquad
    \lambda:=\min\left\{\vartheta_{\min}\tau,\frac12\right\}.
\end{equation*}
Fix a resource $i$ with $v_i>0$. The consumption $(Wy)_i$ is a sum of at most
$K$ terms $w_{ja,i}y_{ja}$ and is at least $v_i$. Hence, some $(j,a)$ with
$a\in\mathcal A_j$ satisfies $w_{ja,i}y_{ja}\ge v_i/K$; in particular,
$w_{ja,i}\ge\omega$. The update $Q_j$ reduces $y_{ja}$ by
$\min\{y_{ja},(W_j^Tv)_a/L_j\}$, and $(W_j^Tv)_a\ge w_{ja,i}v_i$. It therefore
reduces the consumption of resource $i$ by at least
\begin{equation*}
    w_{ja,i}\min\left\{y_{ja},\frac{w_{ja,i}v_i}{L_j}\right\}
    \ge\min\left\{\frac{v_i}{K},\frac{\omega^2v_i}{L_{\max}}\right\}
    \ge\tau v_i,
\end{equation*}
so the overuse of resource $i$ after $Q_j$ is at most $(1-\tau)v_i$. After
$Q_{j'}$ with $j'\ne j$, it is at most $v_i$, and a resource with $v_i=0$ is not
overused after any $Q_{j'}$. Let $I$ have distribution $\vartheta$, and let
$Z_I:=\norm{\pos{WQ_Iy-c}}_1$. Since the type $j$ associated with an overused
resource $i$ is updated with probability $\vartheta_j\ge\vartheta_{\min}$, the
expected overuse of resource $i$ is at most $(1-\vartheta_{\min}\tau)v_i$. Summing over resources and
using $0\le Z_I\le\zeta$, we obtain
\begin{equation}\label{eq:correction-contraction}
    \E_IZ_I\le(1-\lambda)\zeta,
    \qquad
    \E_IZ_I^{2k}\le\zeta^{2k-1}\E_IZ_I\le(1-\lambda)\zeta^{2k}.
\end{equation}
The type associated with a resource may differ across resources; this does not
matter, because the expectation is taken separately for each resource before
the sum over resources.

\medskip\noindent\emph{Step 2: The reward term and the change of state.}
Let $\kappa\in(0,1]$. The vectors $\PG_{c,\eta}^{\kappa,j}(y)$ and $Q_jy$ differ
only in the block of type $j$, where the vectors being projected differ by
$\kappa r_j/L_j$. Since the Euclidean projection is nonexpansive,
$\norm{[\PG_{c,\eta}^{\kappa,j}(y)]_j-[Q_jy]_j}_2\le\kappa\norm{r_j}_2/L_j$.
The positive part is $1$-Lipschitz in each coordinate, and
$\norm\cdot_1\le\sqrt d\norm\cdot_2$ on $\R^d$. Therefore,
\begin{equation}\label{eq:reward-violation-perturbation}
    \norm{\pos{W\PG_{c,\eta}^{\kappa,j}(y)-c}}_1\le Z_j+A_v\kappa,
    \qquad
    A_v:=\sqrt d\max_{j\in[m]}\frac{\norm{W_j}_2\norm{r_j}_2}{L_j}.
\end{equation}
Now let $t<T$, and apply \eqref{eq:reward-violation-perturbation} at the state
$(c_t,\eta_t,y_t^0)$ with $\kappa=\kappa_t$ and $j=I_t$, so that its left side
is $\widehat\zeta_t$; here $Z_{I_t}$ is defined at this state. Since
$\gamma<1$, we have $\ell_t^{-1}\le\ell_t^{-\gamma}=\kappa_t$, and
Lemma~\ref{lem:target-drift} gives
\begin{equation*}
    \zeta_{t+1}\le Z_{I_t}+(A_v+D_v)\kappa_t.
\end{equation*}
Choose $\delta>0$ such that $(1+\delta)^{2k-1}(1-\lambda)=1-\lambda/2$, which is
possible because $0<\lambda\le1/2$, and let $\theta:=1-\lambda/2$. The state of
period $t$ is $\mathcal F_{t-1}$-measurable, and $I_t$ has distribution
$\vartheta$ given $\mathcal F_{t-1}$ by \eqref{eq:block-distribution}. Therefore, \eqref{eq:power-young} with $n=2k$ and
\eqref{eq:correction-contraction} give
\begin{equation}\label{eq:violation-recursion}
    \E\bigl[\zeta_{t+1}^{2k}\mid\mathcal F_{t-1}\bigr]
    \le\theta\zeta_t^{2k}+C_k'\kappa_t^{2k},
    \qquad
    C_k':=(1+\delta^{-1})^{2k-1}(A_v+D_v)^{2k}.
\end{equation}

\medskip\noindent\emph{Step 3: Solving the recursion.}
Taking expectations in \eqref{eq:violation-recursion} and iterating gives
\begin{equation*}
    \E\zeta_t^{2k}
    \le\theta^{t-1}\E\zeta_1^{2k}
       +C_k'\sum_{u=1}^{t-1}\theta^{t-1-u}\ell_u^{-2k\gamma}.
\end{equation*}
With the initial fractions $q_{1,j}=0$, we have $\zeta_1=0$; for any initial
fractions, $\zeta_1\le\norm{\bar w}_1$ by \eqref{eq:capacity-clipping}. Fix $t$
and let $\ell:=\ell_t$ and $b:=2k\gamma$. Since $|\ell_u-\ell_t|\le|u-t|$, every
$u$ with $t-u\le\ell/2$ satisfies $\ell_u\ge\ell/2$, so these terms of the sum
contribute at most $C_k'2^b\ell^{-b}/(1-\theta)$. The remaining terms satisfy
$t-1-u>\ell/2-1$ and $\ell_u\ge1$, so they contribute at most
$C_k'\theta^{\ell/2-1}/(1-\theta)$. Finally,
$\theta^{t-1}\le\theta^{\ell-1}\le\theta^{\ell/2-1}$ because $t\ge\ell$. Since
$\sup_{n\ge1}n^b\theta^{n/2-1}<\infty$, each of the three terms is at most a
constant multiple of $\ell^{-b}=\kappa_t^{2k}$, uniformly in $t$ and $T$. This
proves \eqref{eq:violation-moments}.
\end{proof}

\subsection{Proof of Lemma~\ref{lem:moment-contraction}}
\label{app:moment-contraction}

\begin{proof}
Throughout the proof, $C$ denotes a finite constant that depends only on $k$,
the problem instance, the constants $L_j$, and $\vartheta$, and whose value may
change from
one display to the next. Write $\PG^j:=\PG_{c,\eta}^{\kappa,j}$, let
$v:=\pos{Wy-c}$, and let $g_j:=\kappa r_j-W_j^Tv$ for $j\in[m]$. Since
$\norm v_2\le\norm v_1=\zeta$,
\begin{equation}\label{eq:gradient-block-bound}
    \norm{g_j}_2\le\kappa\norm{r_j}_2+\norm{W_j}_2\zeta,
    \qquad j\in[m].
\end{equation}

\medskip\noindent\emph{Step 1: The expected squared distance.}
Choose $z\in\mathcal O(c,\eta)$ nearest to $y$ in the norm $\norm\cdot_{L/\vartheta}$,
so that $\norm{y-z}_{L/\vartheta}=e$. Since $z_j\in\mathcal Y_j(\eta_j)$ and the
projection onto $\mathcal Y_j(\eta_j)$ is nonexpansive,
\begin{equation*}
    \norm{[\PG^jy]_j-z_j}_2^2
    \le\norm{y_j+g_j/L_j-z_j}_2^2
    =\norm{y_j-z_j}_2^2+\frac2{L_j}(y_j-z_j)^Tg_j+\frac{\norm{g_j}_2^2}{L_j^2}.
\end{equation*}
The vector $\PG^jy$ differs from $y$ only in the block of type $j$, whose weight
in $\norm\cdot_{L/\vartheta}^2$ is $L_j/\vartheta_j$. Since $I=j$ with probability $\vartheta_j$, the
weights cancel in the expectation:
\begin{equation*}
    \E_I\norm{\PG^Iy-z}_{L/\vartheta}^2-e^2
    =\sum_{j=1}^mL_j\Bigl(\norm{[\PG^jy]_j-z_j}_2^2-\norm{y_j-z_j}_2^2\Bigr)
    \le2(y-z)^T\bigl(\kappa r-W^Tv\bigr)
       +\sum_{j=1}^m\frac{\norm{g_j}_2^2}{L_j}.
\end{equation*}
Since $z\in\mathcal O(c,\eta)$, we have $r^Tz=V_r(c,\eta)$ and $Wz\le c$.
For each resource $i$ with $v_i>0$, the latter gives
$(W(y-z))_i\ge(Wy)_i-c_i=v_i$, so $v^TW(y-z)\ge\norm v_2^2\ge0$.
Lemma~\ref{lem:reward-feasibility} gives
$r^Ty-V_r(c,\eta)\le-\sigma e+H_4\zeta$, and $2\kappa\zeta\le\kappa^2+\zeta^2$.
Together with \eqref{eq:gradient-block-bound}, these bounds yield
\begin{equation}\label{eq:squared-distance-drift}
    \E_I\norm{\PG^Iy-z}_{L/\vartheta}^2-e^2
    \le-2\sigma\kappa e+C_1(\kappa^2+\zeta^2),
\end{equation}
where
$C_1:=H_4+2L_{\min}^{-1}\max\{\norm{r}_2^2,\sum_{j=1}^m\norm{W_j}_2^2\}$.
Because $\Pi_{\mathcal Y_j(\eta_j)}(y_j)=y_j$, the same
nonexpansiveness also gives the pathwise bound
\begin{equation}\label{eq:update-displacement}
    \norm{\PG^jy-y}_{L/\vartheta}
    =\sqrt{\frac{L_j}{\vartheta_j}}\,\norm{[\PG^jy]_j-y_j}_2
    \le\frac{\norm{g_j}_2}{\sqrt{\vartheta_jL_j}}
    \le A_0(\kappa+\zeta),
    \qquad
    A_0:=\max_{j\in[m]}\frac{\norm{r_j}_2+\norm{W_j}_2}{\sqrt{\vartheta_jL_j}}.
\end{equation}

\medskip\noindent\emph{Step 2: The $2k$th moment.}
By \eqref{eq:squared-distance-drift}, the expected squared distance decreases
by an amount that is linear in $e$, not proportional to $e^2$, because the error
bound in Lemma~\ref{lem:reward-feasibility} is linear. We therefore pass to a
higher moment, for which the bound $e\le 2M_1$ turns this decrease into a
multiplicative one. For vectors $u$ and $h$ in an inner product space,
\begin{equation}\label{eq:norm-taylor}
    \norm{u+h}^{2k}
    \le\norm u^{2k}+2k\norm u^{2k-2}\langle u,h\rangle
       +C_k^\circ\bigl(\norm u^{2k-2}\norm h^2+\norm h^{2k}\bigr),
    \qquad
    C_k^\circ:=k(2k-1)2^{2k-3}.
\end{equation}
Indeed, the function $\phi(u)=\norm u^{2k}$ has gradient $2k\norm u^{2k-2}u$ and
Hessian $2k\norm u^{2k-2}I+2k(2k-2)\norm u^{2k-4}uu^T$, whose operator norm is
at most $2k(2k-1)\norm u^{2k-2}$. Taylor's formula with integral remainder along
the segment from $u$ to $u+h$ bounds the remainder by
$k(2k-1)(\norm u+\norm h)^{2k-2}\norm h^2$, and
$(\norm u+\norm h)^{2k-2}\le2^{2k-3}(\norm u^{2k-2}+\norm h^{2k-2})$.

We apply \eqref{eq:norm-taylor} in the inner product associated with
$\norm\cdot_{L/\vartheta}$, with $u=y-z$ and $h=\PG^Iy-y$, so that $u+h=\PG^Iy-z$ and
$\norm u_{L/\vartheta}=e$. Since
$2\langle u,h\rangle=\norm{u+h}^2-\norm u^2-\norm h^2\le\norm{u+h}^2-e^2$,
\eqref{eq:squared-distance-drift} bounds the expectation of the linear term. By
\eqref{eq:update-displacement}, $\norm h_{L/\vartheta}^2\le2A_0^2(\kappa^2+\zeta^2)$
and $\norm h_{L/\vartheta}^{2k}\le2^{2k-1}A_0^{2k}(\kappa^{2k}+\zeta^{2k})$. Since the
distance from $\PG^Iy$ to $\mathcal O(c,\eta)$ is at most
$\norm{\PG^Iy-z}_{L/\vartheta}$, we obtain
\begin{equation}\label{eq:raw-moment-drift}
    \E_I\dist_{L/\vartheta}^{2k}\bigl(\PG^Iy,\mathcal O(c,\eta)\bigr)
    \le e^{2k}-2k\sigma\kappa e^{2k-1}
       +Ce^{2k-2}(\kappa^2+\zeta^2)
       +C(\kappa^{2k}+\zeta^{2k}).
\end{equation}
Since $0\le e\le 2M_1$ by \eqref{eq:reduced-diameter}, the second term on the
right is at most $-(k\sigma/M_1)\kappa e^{2k}$. For $\delta>0$ and $b\ge0$,
Young's inequality with the conjugate exponents $k/(k-1)$ and $k$ gives
\begin{equation}\label{eq:young-moment-absorption}
    e^{2k-2}b^2\le\delta\kappa e^{2k}+\delta^{-(k-1)}\frac{b^{2k}}{\kappa^{k-1}}.
\end{equation}
We apply \eqref{eq:young-moment-absorption} with $b=\kappa$ and with $b=\zeta$,
choosing $\delta$ so small that the two resulting terms $C\delta\kappa e^{2k}$
together are at most $(k\sigma/(2M_1))\kappa e^{2k}$. Since $\kappa\le1$, we also
have $\kappa^{2k}\le\kappa^{k+1}$ and $\zeta^{2k}\le\zeta^{2k}/\kappa^{k-1}$.
These bounds prove \eqref{eq:moment-contraction} with
$\alpha_k:=\min\{k\sigma/(2M_1),1/2\}$ and a constant $A_k$ that depends only on
$k$, the problem instance, the constants $L_j$, and $\vartheta$.
\end{proof}

\section{Proof of Proposition~\ref{prop:arrival-models}}\label{app:arrival-models}

We first state a concentration bound for independent types and then verify
Assumptions~\ref{ass:forecast-change} and~\ref{ass:forecast-accuracy} for each
model.

\begin{lemma}\label{lem:independent-concentration}
Let $X_1,\ldots,X_n$ be independent types in $[m]$ with
$\Prob\{X_u=j\}=q_{u,j}$, and let $\bar q:=n^{-1}\sum_{u=1}^nq_u$. Then, for
every $a>0$,
\begin{equation*}
    \Prob\left(
        \norm{\frac1n\sum_{u=1}^n\mathbf e_{X_u}-\bar q}_2>a
    \right)
    \le2m\exp\left(-\frac{2na^2}{m}\right).
\end{equation*}
\end{lemma}

\begin{proof}
If the Euclidean norm exceeds $a$, then some coordinate exceeds $a/\sqrt m$ in
absolute value. For a fixed type $j$, the indicators $\one\{X_u=j\}$,
$u=1,\ldots,n$, are independent and take values in $[0,1]$, and their average
has expected value $\bar q_j$. Hoeffding's inequality \citep{hoeffding_1963}
therefore bounds the probability that this average deviates from $\bar q_j$ by
more than $a/\sqrt m$ by $2\exp(-2na^2/m)$. A union bound over the $m$ types
completes the proof.
\end{proof}

\begin{proof}[Proof of Proposition~\ref{prop:arrival-models}]
We use two elementary facts. First, any two probability vectors $q$ and $q'$ in
$\R^m$ satisfy $\norm{q-q'}_2\le\sqrt2$; in particular,
$\norm{\mathbf e_j-q}_2\le\sqrt2$. Second, for $t<T$, we have $s_t\ge2$ and
$s_t\ge\ell_t$, so $s_t-1\ge s_t/2\ge\ell_t/2$.

\medskip\noindent\emph{Model 1.} Since
$\pi_{t+1}-\pi_t=(\mathbf e_{J_t}-\pi_t)/(t+m)$ and $t+m>\ell_t$, we have
$\norm{\pi_{t+1}-\pi_t}_2\le\sqrt2/\ell_t$, so
Assumption~\ref{ass:forecast-change} holds with $D=2$. Let $p_0:=p_{\min}/2$,
fix $\rho>0$, and let $\rho':=\min\{\rho,p_{\min}\}$. If
$\norm{\pi_t-p}_2\le\rho'/2$ and $\norm{\nu_t-p}_2\le\rho'/2$, then
$\norm{\pi_t-\nu_t}_2\le\rho'\le\rho$, and every coordinate of $\pi_t$ is at
least $p_{\min}-\rho'/2\ge p_0$. Hence, the probability in
Assumption~\ref{ass:forecast-accuracy} is at most
$\Prob(\norm{\nu_t-p}_2>\rho'/2)+\Prob(\norm{\pi_t-p}_2>\rho'/2)$.
Lemma~\ref{lem:independent-concentration}, applied to the $s_t$ types
$J_t,\ldots,J_T$, bounds the first term by $2m\exp(-s_t\rho'^2/(2m))$. For the
second term, let $n:=t-1$ and suppose that $n\ge4m/\rho'$. Let $\widehat p_n$ be
the empirical distribution of $J_1,\ldots,J_n$, and let $u_m$ be the uniform
distribution on $[m]$. Then $\pi_t=\frac n{n+m}\widehat p_n+\frac m{n+m}u_m$.
Since $\norm{u_m-p}_2^2=\norm p_2^2-1/m\le1$ and $m/(n+m)\le\rho'/4$, we have
$\norm{\pi_t-p}_2\le\norm{\widehat p_n-p}_2+\rho'/4$, and
Lemma~\ref{lem:independent-concentration} gives
\begin{equation*}
    \Prob\bigl(\norm{\pi_t-p}_2>\rho'/2\bigr)
    \le\Prob\bigl(\norm{\widehat p_n-p}_2>\rho'/4\bigr)
    \le2m\exp\left(-\frac{n\rho'^2}{8m}\right).
\end{equation*}
At most $\lceil4m/\rho'\rceil$ periods have $t-1<4m/\rho'$, and in these periods
we bound the probability by one. Summing the two geometric series shows that
Assumption~\ref{ass:forecast-accuracy} holds with
\begin{equation*}
    \Lambda(\rho)
    :=\frac{4m}{\rho'}+1+\frac{4m}{\exp(\rho'^2/(8m))-1}.
\end{equation*}

\medskip\noindent\emph{Model 2.} For $t<T$, we have
$\bar p_{t+1}=(s_t\bar p_t-p_t)/(s_t-1)$, so
$\bar p_{t+1}-\bar p_t=(\bar p_t-p_t)/(s_t-1)$ and
$\norm{\bar p_{t+1}-\bar p_t}_2\le\sqrt2/(s_t-1)\le2\sqrt2/\ell_t$. Thus,
Assumption~\ref{ass:forecast-change} holds with $D=2\sqrt2$. Every coordinate of
$\bar p_t$ is at least $p_{\min}=p_0$, so the second event in
Assumption~\ref{ass:forecast-accuracy} never occurs. The vector $\nu_t$ is the
empirical distribution of the $s_t$ independent types $J_t,\ldots,J_T$, whose
distributions have the average $\bar p_t$.
Lemma~\ref{lem:independent-concentration} gives
$\Prob(\norm{\bar p_t-\nu_t}_2>\rho)\le2m\exp(-2s_t\rho^2/m)$, and summing over
$t$ shows that Assumption~\ref{ass:forecast-accuracy} holds with
$\Lambda(\rho):=2m/(\exp(2\rho^2/m)-1)$.

\medskip\noindent\emph{Model 3.} For $t<T$, the forecast
\eqref{eq:history-forecast} satisfies
$\eta_{t+1}-\eta_t=(\eta_t-\mathbf e_{J'_t})/(s_t-1)$, so
Assumption~\ref{ass:forecast-change} holds with $D=2\sqrt2$ as in Model 2. Let
$p_0:=p_{\min}/2$, fix $\rho>0$, and let $\rho':=\min\{\rho,p_{\min}\}$. If
$\norm{\eta_t-\bar p_t}_2\le\rho'/2$ and $\norm{\nu_t-\bar p_t}_2\le\rho'/2$,
then $\norm{\eta_t-\nu_t}_2\le\rho$, and every coordinate of $\eta_t$ is at
least $p_{\min}-\rho'/2\ge p_0$. The vectors $\eta_t$ and $\nu_t$ are the
empirical distributions of the $s_t$ independent types
$J'_t,\ldots,J'_T$ and $J_t,\ldots,J_T$, respectively, and in both cases the
distributions have the average $\bar p_t$. By
Lemma~\ref{lem:independent-concentration} and a union bound, which does not
require the two paths to be independent, the probability in
Assumption~\ref{ass:forecast-accuracy} is at most
$4m\exp(-s_t\rho'^2/(2m))$. Summing over $t$ shows that
Assumption~\ref{ass:forecast-accuracy} holds with
$\Lambda(\rho):=4m/(\exp(\rho'^2/(2m))-1)$.

\medskip\noindent\emph{Model 4.} Assumption~\ref{ass:forecast-change} holds with
$D=2$ as in Model 1. The pairs $X_u:=(S_u,J_u)$ form a Markov chain on the
finite set $\{(s,j)\in\cS\times[m]:p_j(s)>0\}$ with transition probabilities
$P(s,s')p_{j'}(s')$ from $(s,j)$ to $(s',j')$, where $P$ is the transition
matrix of the modulating chain. This chain is irreducible because $P$ is, and
its stationary distribution assigns probability $\varpi_sp_j(s)$ to $(s,j)$.
Let $H$ be the larger of $1$ and the maximum expected hitting time of this
chain, as defined by \citet{moulos_2020}; $H$ depends only on $P$ and the
probabilities $p_j(s)$. Fix $a>0$, a type $j$, and $n$ consecutive periods
$u=t,\ldots,t+n-1$. Theorem~1 of \citet{moulos_2020}, applied to the chain
started in period $t$ and the indicator of $J_u=j$, gives
$\Prob(|\sum_u(\one\{J_u=j\}-\Prob\{J_u=j\})|\ge x)\le2\exp(-2x^2/(nH^2))$
for $x>0$ and every distribution of $X_t$. The hitting-time argument in its
proof also shows that the expected sums $\sum_u\Prob\{J_u=j\}$ for two
distributions of $X_t$ differ by at most $H$; comparing with the stationary
distribution, under which the sum is $n\bar p_j$, gives
$|\sum_u(\Prob\{J_u=j\}-\bar p_j)|\le H$. If $n\ge2H\sqrt m/a$, then
$H\le na/(2\sqrt m)$, and a union bound over the types, as in
Lemma~\ref{lem:independent-concentration}, gives
\begin{equation}\label{eq:markov-concentration}
    \Prob\left(
        \norm{\frac1n\sum_{u=t}^{t+n-1}\mathbf e_{J_u}-\bar p}_2>a
    \right)
    \le2m\exp\left(-\frac{na^2}{2mH^2}\right),
    \qquad n\ge\frac{2H\sqrt m}a.
\end{equation}

Let $\bar p_{\min}:=\min_{j\in[m]}\bar p_j$, $p_0:=\bar p_{\min}/2$, fix
$\rho>0$, and let $\rho':=\min\{\rho,\bar p_{\min}\}$ and
$N:=(4m+8H\sqrt m)/\rho'$. As in Model 1, with $p$ replaced by $\bar p$, the
probability in Assumption~\ref{ass:forecast-accuracy} is at most
$\Prob(\norm{\nu_t-\bar p}_2>\rho'/2)+\Prob(\norm{\pi_t-\bar p}_2>\rho'/2)$.
If $s_t\ge N$, then \eqref{eq:markov-concentration} with $a=\rho'/2$, applied
to the periods $t,\ldots,T$, bounds the first term by
$2m\exp(-s_t\rho'^2/(8mH^2))$. If $n:=t-1\ge N$, then
$\norm{\pi_t-\bar p}_2\le\norm{\widehat p_n-\bar p}_2+\rho'/4$ as in Model 1,
and \eqref{eq:markov-concentration} with $a=\rho'/4$, applied to the periods
$1,\ldots,n$, bounds the second term by $2m\exp(-n\rho'^2/(32mH^2))$. At most
$2N+2$ periods have $s_t<N$ or $t-1<N$, and in these periods we bound the
probability by one. Summing the geometric series shows that
Assumption~\ref{ass:forecast-accuracy} holds with
\begin{equation*}
    \Lambda(\rho)
    :=2N+2
      +\frac{2m}{\exp(\rho'^2/(8mH^2))-1}
      +\frac{2m}{\exp(\rho'^2/(32mH^2))-1}.
\end{equation*}
In each model, $\Lambda(\rho)$ does not depend on $T$.
\end{proof}

\end{document}